\documentclass[twoside,11pt]{article}
\usepackage{jair, theapa, rawfonts}

\usepackage{amsmath}
\usepackage{amsthm}
\usepackage{mathtools}
\usepackage{amssymb}
\usepackage{enumerate}
\usepackage{bm}
\usepackage{mathrsfs}
\usepackage{complexity}

\usepackage{tikz}
\usetikzlibrary{positioning}						
\usetikzlibrary{arrows}						
\usetikzlibrary{decorations.pathreplacing}
\usetikzlibrary{graphs, graphs.standard}
\usetikzlibrary{positioning,calc}

\newtheorem{theorem}{Theorem}

\newtheorem{lemma}[theorem]{Lemma}
\newtheorem{proposition}[theorem]{Proposition}
\newtheorem{conjecture}{Conjecture}
\theoremstyle{definition}
\newtheorem{definition}{Definition}
\AfterEndEnvironment{definition}%
	{\noindent\ignorespaces}
\AfterEndEnvironment{theorem}%
	{\noindent\ignorespaces}
\AfterEndEnvironment{lemma}%
	{\noindent\ignorespaces}
\AfterEndEnvironment{proof}%
	{\noindent\ignorespaces}
\AfterEndEnvironment{corollary}%
	{\noindent\ignorespaces}
\AfterEndEnvironment{proposition}%
	{\noindent\ignorespaces}

\newcommand{\NN}{\mathbb{N}}				

\renewcommand{\emptyset}{\varnothing}			
\newcommand{\dom}{\mathrm{dom}}			
\newcommand{\FF}{\mathcal{F}}				
\renewcommand{\A}{\mathbf{A}}				

\newcommand{\pc}{\text{p.c.}}					
\newcommand{\DTM}{\mathtt{DTM}}			
\newcommand{\IM}{\mathtt{IM}}				
\newcommand{\SEM}{\mathtt{SEM}}			
\newcommand{\SIM}{\mathtt{SIM}}				
\newcommand{\CM}{\mathtt{CM}}				
\newcommand{\MM}{\mathtt{MM}}				
\newcommand{\DIM}{\mathtt{DIM}}				
\renewcommand{\TIME}{\mathrm{TIME}}		
\renewcommand{\SPACE}{\mathrm{SPACE}}		

\newcommand{\T}{\mathtt{T}}					
\renewcommand{\P}{\mathcal{P}}				
\newcommand{\sem}{\mathrm{sem}}				
\newcommand{\syn}{\mathrm{syn}}				
\newcommand{\ind}{\mathrm{ind}}				
\newcommand{\halt}{\mathrm{halt}}				
\renewcommand{\class}{\mathrm{classified}}		
\newcommand{\unclass}{\mathrm{unclassified}}		

\newcommand{\defeq}{\coloneqq}				
\newcommand{\eref}[1]{Eq.~(\ref{#1})}			
\newcommand{\neref}[1]{(\ref{#1})}				
\newcommand{\sref}[1]{Sec.~\ref{#1}}			
\newcommand{\lemref}[1]{Lemma~\ref{#1}}		
\newcommand{\propref}[1]{Proposition~\ref{#1}}	
\newcommand{\thmref}[1]{Theorem~\ref{#1}}		
\newcommand{\conjref}[1]{Conjecture~\ref{#1}}	

\jairheading{-}{-}{-}{-}{-} 					
\ShortHeadings{On Formally Undecidable Traits of Intelligent Machines}{M. Fox}
\firstpageno{1}

\begin{document}

\title{On Formally Undecidable Traits of Intelligent Machines}

\author{\name Matthew Fox \email matthew.fox@colorado.edu \\
\addr Department of Physics, University of Colorado, Boulder, CO}

\maketitle

\begin{abstract}
Building on work by \citeA{Alfonseca21}, we study the conditions necessary for it to be logically possible to prove that an arbitrary artificially intelligent machine will exhibit certain behavior. To do this, we develop a formalism like---but mathematically distinct from---the theory of formal languages and their properties. Our formalism affords a precise means for not only talking about the traits we desire of machines (such as them being intelligent, contained, moral, and so forth), but also for detailing the conditions necessary for it to be logically possible to decide whether a given arbitrary machine possesses such a trait or not. Contrary to \citeS{Alfonseca21} results, we find that Rice's theorem from computability theory cannot in general be used to determine whether an arbitrary machine possesses a given trait or not. Therefore, it is not necessarily the case that deciding whether an arbitrary machine is intelligent, contained, moral, and so forth is logically impossible.
\end{abstract}

\section{Introduction}
\label{sec:Introduction}


In pursuit of what some consider the goal of artificial intelligence (AI) research---namely, ``to understand the principles underlying intelligent behavior and to build those principles into machines that can then exhibit such behavior'' \cite{Russell17,Russell20}---it is the staunch opinion of the author, as well as authors of markedly higher esteem \cite{Bostrom14,Hawking14}, that there is no matter more important than the possibility that we might one day succeed.

Indeed, it is often argued that ``success'' in this goal, while in no way certain, may entail an unimaginably positive future---a future saturated, for example, by incomprehensibly grand healthcare \cite{Davenport19,Sutskever23}, transportation \cite{Wilbur23}, and education \cite{Kamalov23} that is due to, perhaps, unspeakable breakthroughs in biology \cite{Jumper21}, mathematics \cite{Fawzi22,Strogatz22}, physics \cite{Degrave22}, natural language processing \cite{Achiam23}, and so much more \cite{Grace24}. Ultimately, these arguments make it plain that next-generation machine learning (ML) and other AI models possess the potential to be overwhelmingly beneficial for humanity. However, on the axis of good and evil, intelligence lies orthogonal \cite{Bostrom14}. Thus, one should not dismiss the possibility that these same next-generation models simultaneously possess the potential to be overwhelmingly devastating for humanity. 

Outside often harmless quirks such as hallucinations in large language models (LLMs) like GPT-4 \cite{Huang23}, current trends toward autonomy in evermore complex and high-stakes environments suggest that serious accidents and risks are possible \cite{Amodei16,Hendrycks22}. Indeed, several consequential accidents have already occurred, from the infamous fiasco at the Knight Capital Group in which a rogue trading algorithm lost the company \$460 million dollars in just forty-five minutes \cite{Tabbaa18}, to the more recent and deplorable incident at General Motor's Cruise, where an autonomous vehicle dragged a critically injured pedestrian six meters forward while performing a pull-over maneuver \cite{Marshall23}. Since na\"ive resolutions like ``adding qualifiers'' or ``unplugging the AI'' are neither scalable nor logically tenable for controlling advanced systems like next-generation trading algorithms and autonomous vehicles \cite{McCarthy80,Hadfield17}, it is imperative to understand the extent to which AI systems can be accident- and risk-free.

Of course, speculations about AI risks are not new, and the degree to which we humans should actually worry is unclear. Plainly put, it is vague how far contemporary research is from fabricating a generally intelligent machine, and this is despite decades-worth of rapid and groundbreaking advances in compute, data, and algorithms \cite{Sevilla22}. It is thus only rational that ones' credence may apportion against the idea that severe AI risks are possible, let alone imminent. Some even go so far to regard such talk as moonshine, tantamount to tales of Asimovian fiction.\footnote{See \citeS{Ambartsoumean23} article for a comprehensive survey on AI risk skepticism.}

However, unlike other oft-fictionalized events like exceeding the speed of light or breeding a talking animal, the possibility of an AI-induced catastrophe has rallied an unprecedented degree of concern among top researchers \cite{Grace24}. As a result, the relatively nascent field of AI safety is burgeoning, and safety-centered research is now commonplace at many R1 institutions. In addition, governments around the globe are awakening to the solemnity of AI's potential, and are, in consequence, deliberating substantively on regulatory policies, such as the European Union's Artificial Intelligence Act \cite{Yakimova23}.

Thus, questions and concerns surrounding the safety of AI are growing. In part, this is due to our natural desire to fabricate intelligent machines that are not just intelligent, but that are intelligent in such a way that they respect human values. In \citeS{Russell17} words:
\begin{quote}
Up to now, AI has focused on systems that are better at making decisions; but this is not the same as making better decisions. No matter how excellently an algorithm maximizes, and no matter how accurate its model of the world, a machine's decisions may be ineffably stupid, in the eyes of an ordinary human, if its utility function is not well aligned with human values.
\end{quote}
Intelligence, therefore, is not the only trait that we humans ought to desire of our mechanistic progeny. Instead, as Russell intimates, there are other traits we ought to desire of them as well, such as them being contained, moral, corrigible, or whatever it may be that makes the machines intelligent \emph{and} well aligned with human values. Expounding what we want from future intelligent machines, therefore, is paramount to fabricating intelligent systems that are not just better at making decisions, but that are better at making better decisions.

\subsection{Desiderata of Intelligent Machines}

There is thus a problem of desires in AI, and it partitions nicely into an ``easy'' and a ``hard'' problem (though the easy problem is hardly easy). The easy problem is: if we develop intelligent machines in the future (according to some definition of ``intelligent''), what is it that we humans want from them? Is it something grand and idealistic, like an interstellar, utopian, and hedonistic way of life full of riches, galactic travel, and no suffering? Or is it something more prosaic, like a souped-up smartphone running GPT-$42$, or a handy artificial assistant that does your taxes, makes your breakfast, and takes you to the airport?

Answering this \emph{easy problem of desires}---to shortly be contrasted with a much harder problem of desires---is evidently vital if intelligent systems are to be to humanity's benefit (or, in Russell's words, to be ``well aligned with human values''). Equally evident, however, is that the answer to the easy problem is not presently known, and it will plausibly remain unknown for quite some time. Nevertheless, we think it plain that we humans, or at least a subset of us, do desire \emph{something} of the intelligent systems that we develop, for why else would we develop them?\footnote{Even if the answer, however cynical it may seem, is no more substantive than the mere capital gain for those in control, then so be it. Albeit not an admirable one, that is a desire nonetheless.} It is conceivable, therefore, that there is some list of sentences that would constitute at least a partial solution to the easy problem, where by ``partial solution'' we mean those desires for which there is overwhelming human consensus, such as, say, for us humans to maintain some amount of control over the intelligent systems, or, more generically, for us humans to retain our dominance here on earth. Seeking a partial solution, therefore, seems logically, and perhaps even practically tenable, and it in this sense that the easy problem of desires is ``easy''. 


The question that remains, then, is given a (partial) solution to the easy problem, how can we humans ensure that intelligent systems understand and adopt those desires so to act in such a way that they fulfill them? We call this \emph{the hard problem of desires} because, as we shall detail in this paper, depending on the desires we have, it can be logically (and hence practically) impossible to answer, even if our (partial) list of desires is rather short. 

In the context of control, the hard problem of desires entails the famous control problem in AI safety \cite{Russell22}: how can we humans ensure control over intelligent systems? One answer is that we cannot, at least not in general, because if we could, then we could also solve the halting problem \cite{Alfonseca21}. In this way, the control problem is logically, and hence practically, impossible to solve. While we somewhat disagree with this conclusion (\sref{sec:othertraitsofIM}), we underscore that our results generalize many of the techniques used by \citeA{Alfonseca21}. This is made possible by the fact that the hard problem itself is a direct generalization of the control problem, as it includes any desire whatsoever that we humans might have of intelligent systems.

Obviously, insofar as any given intelligent system is capable of satisfying a particular human desire, obtaining certainty that that system actually satisfies that desire is a tremendously high bar, because it necessitates a formal mathematical proof based on a description of the system and the desire. While for certain systems based on particular computational models (such as inverse reinforcement learning) and for certain desires (such as an ability for us humans to turn a system off) this may well be possible \cite{Hadfield17}, it is natural to ask if certainty in general is possible, even in principle. It is the purpose of this paper to explore this question, and to understand the extent to which a solution to the hard problem of desires is logically tenable.

In addition to knowing the (partial) solution to the easy problem of desires, however, at least two other fundamental issues obstruct us from directly addressing the hard problem. The first is that we require the (partial) solution to the easy problem to be said precisely enough so to be faithfully interpreted by machines. This in itself is a sort of ``intermediate problem of desires'', which the philosopher Bertrand Russell \citeyear{Russell72} would likely argue is insuperable:
\begin{quote}
Everything is vague to a degree you do not realize till you have tried to make it precise, and everything precise is so remote from everything that we normally think, that you cannot for a moment suppose that is what we really mean when we say what we think.
\end{quote}

Here, Russell's view is undoubtedly informed by \citeS{Quine51} famous essay \emph{Two Dogmas of Empiricism}, which suggests that not only is the intermediate problem hard, but so too is the easy problem, since linguistical issues surrounding the meaning of words abound, particularly in the context of synonymy and the distinction between analytic and synthetic truths. Nevertheless, Stuart Russell's \citeyear{Russell19} proposal of having machines learn our desires may be a way forward here, as it seemingly circumvents our having to explicitly articulate and encode our values into computers.

The second issue is that when it comes to the actions of machines inherently more intelligent than ourselves, it can be hard to say anything at all with confidence. After all, who are we to say what a system far more intelligent than the collective human species will and will not be able to do? Historically, indeed even dating back to the time of Turing, there has been an abundance of pronouncements by prominent computer scientists to the effect that ``no machine will ever do X'', where X is, for example, ``invent logical proofs'' \cite{McCorduck79} or ``beat a 9-dan professional at Go'' \cite{Johnson97}.\footnote{See \citeA{Bostrom14} for many more.} But of course, these and many other assertions were all eventually confuted. Thus, if history is any teacher, perhaps the lesson is that for any such X, the pronouncement that ``no machine will ever do X'' is na\"ive at worst and unimaginative at best. However, we believe with high credence that there are two choices of X---two ``disabilities'' as \citeA{Turing50} called them---which are necessarily out of reach of every machine, no matter how intelligent. They are:
\begin{enumerate}[(I)]
\item transcend the laws of physics;
\item compute a function not computable by any Turing machine.
\end{enumerate}

Depending on one's priors, one may find (I) or (II) more obvious than the other. A physicist will most likely find (I) the most evident, basically because intelligent machines, whatever they are, are at the very least physical devices and are ipso facto constrained by physical laws. Otherwise, if one suspects that intelligent machines may in some way transcend the laws of physics, then one is necessarily subscribing to a sort a dualism in which there exist both ``physical'' and ``non-physical'' events. However, akin to standard dualist philosophy, such a view invariably raises an analogue of Princess Elisabeth's interaction problem: how can the non-physical events that intelligent machine witness---events which uphold no physical properties such as mass, charge, location, or velocity---make a physical difference in the world? Douglas Hofstadter and Daniel Dennett emphasize the absurdity of this view in their anthology \emph{The Mind's I} \citeyear{Hofstadter01}, in which they say:
\begin{quote}
For a non-physical event to make a difference, it must make some physical event happen that wouldn't have happened if the non-physical event hadn't happened. But if we found a sort of event whose occurrence had this sort of effect, why wouldn't we decide \emph{for that very reason} that we had discovered a new sort of \emph{physical} event?
\end{quote}
Indeed, the history of physics suggests that any sort of ``unphysical'' thing that comprises intelligent systems, or that stems from them, will almost surely be regarded as part of a new paradigm of physics, tantamount to how the anomalous precession of the perihelion of Mercury was not seen as ``unphysical'', but rather as evidence that the Newtonian paradigm ultimately required revision \cite{Baum97}. All this suggests that (I) is almost surely the case.

What about (II)? Could it ever happen that an intelligent machine computes a function not computable by any Turing machine? To the computer scientist, this may be the obvious ``no'', perhaps in part due to daily exposure to algorithms, which are always simulable by some Turing machine. However, in search of a \emph{fundamental} reason why the answer to (II) is also almost surely no, we now discuss why it is essentially a staple of contemporary theoretical physics that (I) and (II) are logically equivalent.

\subsection{Computability in a Physical World}
\label{sec:computabilityinworld}

Let $\Sigma$ be an alphabet and let $\FF$ be the set of all partial functions from the set $\Sigma^*$ of all finite $\Sigma$-strings to $\Sigma^*$. A basic diagonalization argument proves that $\FF$ is not countable. However, if our means of reasoning about $\FF$ and the functions therein is limited to the finite manipulation of only a finite number of symbols (as is the case with pen and paper calculation and also all computer programming), then we humans---supplemented, perhaps, by unboundedly advanced machines, unlimited time, and all the physical resources in causal connection to us---can at most compute a countable number of the functions in $\FF$. 

Let us call this distinguished subset of functions $\Delta_1$, in accordance with the notation used to label the first level of the arithmetic hierarchy, which is a closely related object \cite{Kozen06}. Thus, $\Delta_1$ contains only those partial functions $\varphi : \Sigma^* \rightarrow \Sigma^*$ for which there is an effective, pen-and-paper procedure whereby a rote worker can deduce $\varphi(\sigma)$ for any $\sigma \in \Sigma^*$ in the domain of $\varphi$. Because of this, we say $\varphi$ is \emph{effectively calculable} if and only if (iff) $\varphi \in \Delta_1$.

Of course, ``effective calculability'' and the associated set $\Delta_1$ are awkwardly informal. It is the purpose of the Church-Turing thesis (CTT), however, to postulate precisely what these notions actually mean:
\begin{quote}
\textbf{Church-Turing Thesis:} A function $\varphi : \Sigma^* \rightarrow \Sigma^*$ is effectively calculable (and hence in $\Delta_1$) iff it is computable on a deterministic Turing machine.
\end{quote}

Despite the ostensible rigor of this statement, it is paramount that the CTT not be thought as something that one can rigorously prove, but rather as something, in \citeS{Turing54} own words, ``between a theorem and a definition''. The CTT, therefore, constitutes a sort of belief, a sort of intuition about what we humans think ``effective calculability'' rigorously means. In this way, the thesis is subject to the more Bayesian methods of science in which abductive, and not strictly deductive means are necessary to substantiate its claim. In the words of Stephen Kleene \citeyear{Kleene52}:
\begin{quote}
While we cannot prove [the CTT], since its role is to delimit precisely an hitherto vaguely conceived totality, we require evidence that it cannot conflict with the intuitive notion which it is supposed to complete; i.e., we require evidence that every particular function which our intuitive notion would authenticate as effectively calculable is [computable].
\end{quote}

It is along these lines that correlations between the physical world and effective calculability invariably creep in. In particular, since as humans our intuitions largely derive from our experiences of the physical world, what is regarded as effectively calculable may in some way correlate with the surrounding physical world. After all, as we detail in \sref{sec:DTMs}, the intuitive picture on which the Turing machine is based is manifestly physical, and all computers that exist today, and that will exist in the future, are necessarily physical devices. Thus, it is reasonable to think that computation is inextricably tied to physical law.

To see that this is plausibly the case, consider an ostensibly different subset of $\FF$ than $\Delta_1$, namely that set of partial functions $\varphi : \Sigma^* \rightarrow \Sigma^*$ for which (intuitively speaking) there exists a physical system---like a box of helium gas, a large register of electrons subject to a particular unitary dynamics, or a series of entangled black holes evolving according to an as-yet-undiscovered theory of quantum gravity---such that for all $\Sigma$-strings $\sigma$ in the domain of $\varphi$, $\varphi(\sigma)$ is computable by the mere physical evolution of the system. We shall call such functions \emph{physically calculable}. As above, this talk is awkwardly imprecise. Thus, like talk of effective calculability, it will take a thesis to raise the level of rigor to something more amenable to mathematical analysis.

The question, then, is what is the right thesis? Intuitively, every effectively calculable function is physically calculable. This follows because Turing machine instructions are finite, and so they are simulable with pen, paper, and a rote worker (all of which are physical systems abiding by the laws of physics). Conversely, insofar as quantum mechanics underlies everything there is in the universe (and we have as yet no serious evidence to the contrary), every physically calculable function is computable by some quantum system. Therefore, every physically calculable function is computable on a quantum computer, because quantum computers can simulate any quantum system \cite{Nielsen11}. But, modulo exponential overheads in computational complexity, quantum computers are simulable by classical computers, and hence by Turing machines \cite{Arora09}. 

Consequently, we are lead to the conclusion that a function is effectively calculable iff it is physically calculable, which is to say that the set $\Delta_1$ of effectively calculable functions \emph{equals} the set of physically calculable functions. The CTT, then, seems to be a postulate not just of computability theory, but of the physical world \cite{Deutsch85}:

\begin{quote}
\textbf{Physical Church-Turing Thesis:} A function $\varphi : \Sigma^* \rightarrow \Sigma^*$ is physically calculable iff it is computable on a deterministic Turing machine.
\end{quote}

In the author's opinion, this elevates $\Delta_1$ to the level of a fundamental constant of nature, on a par with the speed of light $c$, Planck's constant $\hbar$, and Newton's gravitational constant $G$. Notably, the PCTT entails that even those functions computed by the most exotic and esoteric processes in the universe are in $\Delta_1$. For example, consider the ridiculous process of encoding a $\Sigma$-string $\sigma$ into a series of entangled quantum particles, throwing them into a black hole, and then decoding the Hawking radiation at a later time into a string $\gamma$. Altogether, this constitutes a function $\sigma \mapsto \gamma$, and it is a genuine surprise of physics that this function is in $\Delta_1$---a function so imbedded in the fundamental laws of physics.

In summary, then, what undergirds proposition (II) in the preceding subsection is that computation itself is fundamentally physical. Therefore, if in the future an intelligent machine were to bring to bear all the resources of the physical world (encapsulated, perhaps, in a correct---or at least more correct---theory of quantum gravity), then contemporary physics suggests that there is no capacity for it to transcend the computational abilities of Turing machines. Rather, all that it might achieve is to reduce the computational complexity of the problems with which it is tasked, like how quantum computers afford an efficient means for factoring large integers.\footnote{Of course, reducing the computational complexity of a problem is no small feat, for the difference between an $O(2^n)$ algorithm versus an $O(n^2)$ algorithm can literally be a lifetime. Our point here, however, has to do with \emph{effective} (as opposed to \emph{efficient}) computation, and the laws of physics as we know them do not change what is and what is not effectively computable.} Therefore, if intelligent machines are constrained by the laws of physics (and again we see no way for them not to be), then they must also be constrained by the theory of computability.

We note that insofar as the PCTT remains true, this conclusion also goes the other way. That is, if intelligent machines are constrained by the theory of computability, then, by the PCTT, intelligent machines cannot exploit the laws of physics to compute anything not computable by a Turing machine.

In the remainder of this paper we shall not dwell too much more on the constraints imposed on intelligent systems by physical law (in part because the laws of physics, whatever they are, are not yet known, and may in fact never be fully known). Instead, our task here is to assume that intelligent systems are physical devices and then to take special care of the constraints thus imposed by computability theory.

\subsection{Intelligent Machines and Computability Theory}
\label{sec:MIH}

Mirroring \citeA{Good66}, let an \emph{intelligent machine} $\A$ be defined as any physical system (e.g. a machine) that is on a par with many, if not all, of the intellectual activities of any human. Like \citeA{Alfonseca21}, we also assume $\A$ is programmable, so that when it receives an input from the external world (such as the state of the world), it acts on the external world exclusively on the basis of the output of its program. By this definition, it is plain that the ``intelligence'' of $\A$ is inherent to its program, for if $\A$ ran no program, then its ``plan of action'', so to speak, would necessarily be void, in which case $\A$ would not act at all, let alone intelligently.

Let $M_\A$ be the program of $\A$, which on input $\sigma \in \Sigma^*$ computes a function $\varphi_{M_{\A}} : \Sigma^* \rightarrow \Sigma^*$ that details how $\A$ should act on the world. Since $\A$ is by definition a physical system, it holds that $\varphi_{M_\A}$ is physically calculable. Therefore, by the PCTT, $\varphi_{M_\A} \in \Delta_1$, i.e., $\varphi_{M_\A}$ is computable on a deterministic Turing machine.

Therefore, without any loss of generality, we may assume that the program $M_\A$ that $\A$ runs is a deterministic Turing machine. In consequence, for any definition of ``intelligent'', the set $\IM$ of all programs of intelligent machines is but a subset of the set $\DTM$ of all deterministic Turing machines:
\begin{equation}
\IM \subseteq \DTM.
\label{eq:SIMsubDTM}
\end{equation}
We hereafter refer to this inclusion, as well as the necessary assumptions leading up to it (such as the PCTT), as the \emph{Machine Intelligence Hypothesis} (or \emph{MIH}).

It is the purpose of this paper to explore the formal consequences of the MIH. We are particularly interested in the question: what fundamental limitations does \eref{eq:SIMsubDTM} impose on what we can and cannot say about intelligent machines, particularly in the context of the hard problem of desires? Our main results formalize an hitherto intuitive notion---namely, that in designing advanced intelligent systems, worrying about the resources necessary for an algorithm to properly run is paramount, because this affects our ability to prove if a given algorithm will behave in a particular manner (such as it being contained, moral, corrigible, and so forth.).

An approximate outline of this article is as follows. In \sref{sec:Preliminaries}, we introduce the necessary results in computability theory, such as the padding lemma and Rice's theorem, in order to say what the MIH decisively entails. Our main technical results are then developed in \sref{sec:traitsofDTMs}, which are later applied to intelligent machines (for any definition of ``intelligent'') in \sref{sec:applicationtoIMs}. Finally, in \sref{sec:discussion}, we discuss the greater context of our results, such as their relevance in an ever-hastening world of AI research and development.

\section{Preliminaries}
\label{sec:Preliminaries}

In this section, we introduce our notation and review several foundational notions from computability theory, such as Turing machines, partial computable functions, index sets, and Rice's theorem. Any reader familiar with these topics may skip to \sref{sec:traitsofDTMs}, where we develop our central definitions and theorems that are used later in this paper.

\subsection{Deterministic Turing Machines and Partial Computable Functions}
\label{sec:DTMs}

Let $\Sigma$ be an alphabet and let $\Sigma^*$ equal the set of all $\Sigma$-strings. Per convention, given two (unary\footnote{In fact, all partial functions are unary in the sense that for any finite $k \geq 1$, the $k$-fold Cartesian product $(\Sigma \times \Sigma \times \cdots \times \Sigma)^*$ is bijective to $\Sigma^*$ in an effectively calculable way \cite{Kozen06}.}) partial functions $\varphi, \psi: \Sigma^* \rightarrow \Sigma^*$ with domains $\dom(\varphi)$ and $\dom(\psi)$, respectively, we write $\varphi \simeq \psi$ iff $\varphi$ and $\psi$ are equal in the sense of partial functions, i.e., $\dom(\varphi) = \dom(\psi)$ and $\varphi(\sigma) = \psi(\sigma)$ for all $\sigma \in \dom(\varphi)$. Moreover, we write $\varphi(\sigma)\downarrow$ or $\varphi(\sigma) = \downarrow$ iff $\sigma \in \dom(\varphi)$ (in which case $\varphi(\sigma)$ is \emph{defined}) and $\varphi(\sigma)\uparrow$ or $\varphi(\sigma) = \uparrow$ iff $\sigma \not\in \dom(\varphi)$ (in which case $\varphi(\sigma)$ is \emph{undefined}). 

We remark that more general partial functions that map strings over an alphabet $\Gamma$ to strings over a different alphabet $\Omega$ are subsumed by this formalism by simply taking $\Sigma = \Gamma \cup \Omega$. Note, also, that for any alphabet of symbols $\Sigma$ like $\{0,1\}$ or $\{a,b, \dots, z\}$, the set $\Sigma^*$ of all $\Sigma$-strings is bijective to $\NN$. Therefore, one can always talk about arithmetic functions in order to reason about functions over a different alphabet. However, in order to leave the exact alphabet we use unfixed, we hereafter just employ a general alphabet $\Sigma$.

We now summarize the modern definition of what it means to ``compute'' a partial function, which was originally put forth by \citeA{Turing36}. Turing's idea rests on the notion of a \emph{Turing machine}, which he envisaged as follows \cite{Turing48}:

\begin{quote}
[A Turing machine has] an unlimited memory capacity obtained in the form of an infinite tape marked out into squares, on each of which a symbol could be printed. At any moment there is one symbol in the machine; it is called the scanned symbol. The machine can alter the scanned symbol, and its behavior is in part determined by that symbol, but the symbols on the tape elsewhere do not affect the behavior of the machine. However, the tape can be moved back and forth through the machine, this being one of the elementary operations of the machine. Any symbol on the tape may therefore eventually have an innings.
\end{quote}

Today, the Turing machine is generally accepted as the simplest model of computation. Of course, there are a myriad of other computational models, such as G\"odel's $\mu$-recursive functions, Church's $\lambda$-calculus, C programs, and so forth. But as these are all formally equivalent to Turing's simple machine, it is mathematically preferred to just deal with Turing machines---the simplest of the bunch. Thus, while what we derive in this paper is explicitly with respect to Turing machines, implicitly our conclusions apply much more broadly to the intersection of AI and the theory of computation itself.

While abstract, Turing machines afford a mathematical setting in which previously intuitive notions like ``algorithms'' and ``computation'' are made precise. In consequence, we are able to prove formal statements about algorithms. Of course, there are many different types of Turing machines, from two-way, single-tape, deterministic Turing machines that read and write over the binary alphabet, to two-way, multi-tape, probabilistic Turing machines that read in ternary and write in hexadecimal and that always halt in exponential space. In this paper, however, we do not concern ourselves with resource-bounded computations, which is the realm of computational complexity theory. Instead, we only care about the theory of computability, which is generally concerned with those functions that can be computed, no matter how much time or space it takes (so long, of course, that it is finite). For this reason, and the fact that there exists a universal Turing machine in the single-tape, deterministic model, we assume hereafter that all Turing machines are of this type. In terms of computability, the more exotic sounding types of machines like the probabilistic machine above are subsumed by the deterministic, single-tape model, albeit at the expense of a longer runtime. Of course, this must be the case, for otherwise the Church-Turing thesis would fail.

The formal definition of a deterministic Turing machine is as follows.

\begin{definition}
A \emph{two-way, single-tape, deterministic Turing machine} $M$ (hereafter called a \emph{deterministic Turing machine} or \emph{DTM} for short) is a tuple $(Q, \Sigma, \Gamma, \delta)$, where $Q$ is a finite and nonempty set of \emph{states}, $\Sigma$ is the \emph{input alphabet}, $\Gamma \supseteq \Sigma \cup \{\diamond\}$ is the \emph{tape alphabet} (with \emph{blank symbol} $\diamond$), and $\delta$ is the deterministic \emph{transition function}, which bears the form
\begin{equation}
\delta : Q \backslash \{q_A, q_R\} \times \Gamma \rightarrow Q \times \Gamma \times \{\lhd, \rhd\}.
\label{eq:TuringMachine}
\end{equation}
The state set $Q$ necessarily contains three distinguished states: the \emph{initial state} $q_0$, the \emph{accept state} $q_A$, and the \emph{reject state} $q_R$. Together, $q_A$ and $q_R$ constitute the \emph{halting states} and satisfy $q_A \neq q_R$. Importantly, since the domain of $\delta$ excludes the halting states, $M$ \emph{halts} if and only if it transitions to a halting state. We denote the set of all DTMs by $\DTM$.
\end{definition}

Following Turing, one should picture an initialized DTM $M$ as follows:
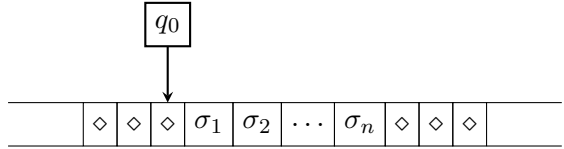
\begin{equation*}
\begin{tikzpicture}[baseline=(current  bounding  box.center), every node/.style={block},
        block/.style={minimum height=1.5em,outer sep=0pt,draw,rectangle,node distance=0pt}]
   \node (A) {$\ldots$};
   \node (B) [left=of A] {$\sigma_2$};
   \node (C) [left=of B] {$\sigma_1$};
   \node (D) [left=of C] {$\diamond$};
   \node (E) [left=of D] {$\diamond$};
   \node (F) [left=of E] {$\diamond$};
   \node (G) [right=of A] {$\sigma_n$};
   \node (H) [right=of G] {$\diamond$};
   \node (I)  [right=of H] {$\diamond$};
   \node (J)  [right=of I] {$\diamond$};
   \node (K) [above = 0.75cm of D,draw=black,thick] {$q_0$};
   \draw[-stealth, thick] (K) -- (D);
   \draw (F.north west) -- ++(-1cm,0) (F.south west) -- ++ (-1cm,0) 
                 (J.north east) -- ++(1cm,0) (J.south east) -- ++ (1cm,0);
\end{tikzpicture}
\end{equation*}
Here, the \emph{finite state control} (the black box above the tape) is initialized in the state $q_0$, and the two-way infinite tape contains the string $\sigma = \sigma_1\sigma_2\dots\sigma_n \in \Sigma^*$, with one letter $\sigma_i \in \Sigma$ per cell, which encodes the computation to be done. Every other cell of the tape contains the blank symbol $\diamond$. The \emph{tape head} of $M$ (black arrow) is initialized so that it points to the blank cell directly to the left of the first symbol in $\sigma$, namely $\sigma_1$. After the initialization, $M$ starts taking \emph{steps}, which are calls to the transition function $\delta$. For example, the first step could be $\delta(q_0, \diamond) = (q_1, \gamma, \rhd)$, which, in the canonical interpretation, means: 
\begin{enumerate}[(i)]
\item the state of $M$ transitions from $q_0$ to some $q_1 \in Q$, 
\item the tape head overwrites $\diamond$ with a (not necessarily new) symbol $\gamma \in \Gamma$, 
\item the tape head moves one cell to the right (as defined by the ``move right'' symbol $\rhd$).
\end{enumerate}
Pictorially, this new configuration is:
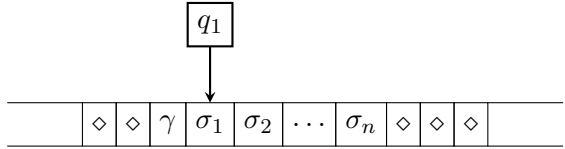
\begin{equation*}
\begin{tikzpicture}[baseline=(current  bounding  box.center), every node/.style={block},
        block/.style={minimum height=1.5em,outer sep=0pt,draw,rectangle,node distance=0pt}]
   \node (A) {$\ldots$};
   \node (B) [left=of A] {$\sigma_2$};
   \node (C) [left=of B] {$\sigma_1$};
   \node (D) [left=of C] {$\gamma$};
   \node (E) [left=of D] {$\diamond$};
   \node (F) [left=of E] {$\diamond$};
   \node (G) [right=of A] {$\sigma_n$};
   \node (H) [right=of G] {$\diamond$};
   \node (I)  [right=of H] {$\diamond$};
   \node (J)  [right=of I] {$\diamond$};
   \node (K) [above = 0.75cm of C,draw=black,thick] {$q_1$};
   \draw[-stealth, thick] (K) -- (C);
   \draw (F.north west) -- ++(-1cm,0) (F.south west) -- ++ (-1cm,0) 
                 (J.north east) -- ++(1cm,0) (J.south east) -- ++ (1cm,0);
\end{tikzpicture}
\end{equation*}

This process continues. If ever $M$ halts, then its output is defined by whatever non-blank string $\gamma_1\gamma_2\dots\gamma_m \in \Sigma^*$ remains on its tape. The map $\sigma_1\sigma_2\dots\sigma_n \mapsto \gamma_1\gamma_2\dots\gamma_m$ is known as the \emph{proper function} of $M$.

\begin{definition}
The \emph{proper function} of a DTM $M$ is the function $\varphi_M$ that $M$ computes when some $\sigma = \sigma_1\sigma_2\dots\sigma_n \in \Sigma^*$ is initially written on its tape. More formally, it is the partial function $\varphi_M : \Sigma^* \rightarrow \Sigma^*$ such that for all $\sigma \in \Sigma^*$,
\begin{equation}
\varphi_M(\sigma) =
\begin{cases}
\gamma &\text{if $M$ halts and the non-blank string $\gamma \in \Sigma^*$ remains on its tape,}\\
\uparrow &\text{if $M$ does not halt or the tape is blank.}
\end{cases}
\end{equation}
Correspondingly, we say a (partial or total) function $\varphi: \Sigma^* \rightarrow \Sigma^*$ is \emph{partial computable} (\emph{\pc})\ iff there is $M \in \DTM$ such that $\varphi_M \simeq \varphi$. By the Church-Turing thesis, the set of all \pc\ functions $\varphi : \Sigma^* \rightarrow \Sigma^*$ is exactly the set of all effectively calculable functions $\Delta_1$, which was outlined in \sref{sec:computabilityinworld}.
\end{definition}

Among the many remarkable features of Turing machines is their universality. That is, there exist \emph{universal Turing machines} in $\DTM$ that, given a description of any other Turing machine $N$, can simulate $N$ on any input. Incidentally, in the context of AI, \citeA{Bostrom14} has argued that any ``superintelligent'' machine must be universal in this sense.

For more on Turing machines, their function, and their universality, we refer the interested reader to any introductory book on computability theory, e.g., \citeS{Sipser13} or \citeS{Robic20}.

\subsection{Indices of Deterministic Turing Machines and Index Sets}

Evidently, $\DTM$ is a denumerable set, and it is straightforward to show that there exists a p.c.\ bijection $\eta : \DTM \rightarrow \NN$ such as the famous G\"odel numbering \cite{Kozen06}. This constitutes an \emph{indexing} of DTMs, also known as an \emph{enumeration} of DTMs. An indexing $\eta$ thus associates to every $M \in \DTM$ a unique positive integer $\eta(M)$ called its \emph{index}. In turn, the index of a DTM $M$ associates a unique index to its proper function:
\begin{equation}
\eta(\varphi_M) \defeq \eta(M).
\end{equation}
It is tempting to think that this, in turn, thus associates an index to \emph{all} the \pc\ functions equal to $\varphi_M$. However, this is not right because while every natural number is the index of exactly one DTM (and hence one \pc\ function), not every \pc\ function is associated with exactly one index. This follows because different DTMs can compute the same \pc\ function. For example, there are many different algorithms for computing the Fibonacci numbers.
\begin{definition}
Let $\eta : \DTM \rightarrow \NN$ be an indexing of DTMs. The \emph{index set} $\ind_\eta(\varphi)$ of a \pc\ function $\varphi : \Sigma^* \rightarrow \Sigma^*$ is the set of indices of DTMs that compute it. That is, 
\begin{equation}
\ind_\eta(\varphi) \defeq \left\{\eta(M) \mid M \in \DTM \land \varphi_M \simeq \varphi\right\}.
\end{equation}
\end{definition}

How large is $\ind_\eta(\varphi)$? That is, how many algorithms are there that compute a given \pc\ function $\varphi$? Since any algorithm can be edited so to implement the original program, but before returning the answer, loop for an integral number of times, it is perhaps not surprising that $|\ind_\eta(\varphi)|$ is \emph{infinite}. In computability theory, this result is known as the \emph{padding lemma} \cite{Robic20}:
\begin{lemma}[Padding Lemma]
For every indexing $\eta$ of DTMs and every \pc\ function $\varphi$, $\ind_\eta(\varphi)$ is a denumerable set.
\label{lem:paddinglemma}
\end{lemma}

A slick proof of this fact follows from Rice's theorem (\thmref{thm:RiceTheorem} below), which can be found in most books on computability theory. 

\subsection{Properties of Partial Computable Functions}

Often, it is natural to consider only those \pc\ functions that abide by some criterion or satisfy some property. For example, as computational complexity theorists do, one may wish to understand those \pc\ functions for which there is a resource-bounded DTM that computes it, such as a DTM that always halts in polynomial time. It is similarly the case in AI, where one may wish to understand those \pc\ functions that are computable by ``intelligent machines'' or ``intelligent moral machines'', which are topics that we explore in \sref{sec:applicationtoIMs}. For now, however, all we seek is a formal means for even talking about distinguishing \pc\ functions.

\begin{definition}
A \emph{\pc\ function property} (or \emph{property} for short) $\P$ is a collection of \pc\ functions in $\Delta_1$. Given an indexing $\eta : \DTM \rightarrow \NN$, the \emph{index set of $\P$} is the union of the index sets of the functions in $\P$, i.e.,
\begin{equation}
\ind_\eta(\P) \defeq \bigcup_{\varphi \in \P} \ind_\eta(\varphi).
\label{eq:indicesofproperties}
\end{equation}
We say a p.c. function $\varphi$ \emph{satisfies} $\P$ iff $\varphi \in \P$. Otherwise, $\varphi \not\in \P$ and $\varphi$ \emph{does not satisfy $\P$}. Finally, a property $\P$ is \emph{trivial} iff either every \pc\ function in $\Delta_1$ satisfies $\P$ (i.e. $\P = \Delta_1$) or no \pc\ function satisfies $\P$ (i.e. $\P = \emptyset$). Otherwise, $\P$ is \emph{nontrivial}.
\end{definition}

For example, let $\RE$ be the collection of all recursively enumerable languages over $\Sigma^*$. Then, for each $L \in \RE$, there is a DTM $M_L$ such that for all $\sigma \in \Sigma^*$,
\begin{equation}
\sigma \in L \iff \varphi_{M_L}(\sigma)\downarrow.
\label{eq:partialcharacteristicfunction}
\end{equation}
Consequently, $L = \dom(\varphi_{M_L})$, so the proper function $\varphi_{M_L} : \Sigma^* \rightarrow \Sigma^*$ constitutes a \emph{partial characteristic function} of $L$ because it computes (and only computes) the ``yes'' instances of the decision problem ``is $\sigma$ in $L$?''. In particular, ``yes'' is entailed by any output whatsoever. Altogether, the set of all \pc\ partial characteristic functions of the recursively enumerable languages defines the property
\begin{equation}
\P_\RE \defeq \bigcup_{L \in \RE} \left\{\varphi \mid \varphi \in \Delta_1 \land \dom(\varphi) = L\right\}.
\end{equation}
It is plain that every \pc\ function is in $\P_\RE$, so $\P_\RE$ is trivial.

As another example, consider any nonempty complexity class $\C$ of \emph{decidable} languages like $\mathsf{P}$, $\NP$, $\BQP$, or $\PSPACE$.
Then, by the definition of a decidable language, for each $L \in \C$, there is a DTM $M_L$ and a distinguished string $\tilde{\sigma} \in \Sigma^*$ such that for all $\sigma \in \Sigma^*$,
\begin{equation}
\varphi_{M_L}(\sigma)\downarrow \quad \land \quad \sigma \in L \iff \varphi_{M_L}(\sigma) = \tilde{\sigma}.
\label{eq:characteristicfunction}
\end{equation}
Thus, the proper function $\varphi_{M_L}$ is total and equals a characteristic function of $L$ because it computes both the ``yes'' and ``no'' instances of the decision problem ``is $\sigma$ in $L$?''. In particular, ``yes'' is entailed by the output $\tilde{\sigma}$ and ``no'' is entailed by any output other than $\tilde{\sigma}$. Altogether, the set of all \pc\ characteristic functions of the languages in $\C$ defines the property
\begin{equation}
\P_\C \defeq \bigcup_{L \in \C} \left\{\varphi \mid \varphi \in \Delta_1 \land \forall\sigma \in \Sigma^* : \varphi(\sigma)\downarrow \land \left(\sigma \in L \iff \varphi(\sigma) = \tilde{\sigma}\right)\right\}.
\end{equation}
Evidently, $\P_\C \neq \emptyset$ (because $\C \neq \emptyset$) and $\P_\C \neq \P_\RE$ (because $\C \neq \RE$ by the existence of undecidable languages such as the halting problem). Therefore, $\P_\C$ is nontrivial.

\subsection{Rice's Theorem}

A natural problem is to determine whether a \pc\ function $\varphi$ satisfies a given property $\P$ or not. To use an example from before, one may wish to understand those \pc\ functions that are computable by ``intelligent moral machines''. Of course, an understanding like this is achievable if given a \pc\ function $\varphi$ one could algorithmically decide if it is computable by an intelligent moral machine or not. Obviously, though, in order to speak formally about this, there needs to be a language that represents the decision problem ``does the \pc\ function $\varphi$ satisfy $\P$?''. Fortunately, as the next definition shows, the requisite pieces to do this are already in their place.

\begin{definition}
Together, a \pc\ function property $\P$ and an indexing $\eta : \DTM \rightarrow \NN$ generate a language over $\NN$:
\begin{equation}
L_\P \defeq \ind_\eta(\P).
\label{eq:propertylanguage}
\end{equation}
We say $\P$ is \emph{decidable} iff $L_{\P}$ is decidable.
\end{definition}

This is a standard definition \cite{Robic20}, and we note that it entails $\varphi \in \P$ iff $\ind_\eta(\varphi) \subseteq L_{\P}$. In consequence, $\varphi \in \P$ iff every \pc\ function equal to $\varphi$ is also in $\P$. In this sense, then, deciding a \pc\ function property reduces down not to interrogating the syntactic structure of a particular \pc\ function, but rather to an interrogation of its more global \emph{semantic} behavior. In other words, deciding a \pc\ function property is independent of understanding exactly \emph{how} a function satisfying that property is constructed. Instead, it merely reduces to understanding the input-output relationship of that function. This syntactic versus semantic distinction is important, and it is something that we will come back to in \sref{sec:traitsofDTMs}.

By the padding lemma, for every nontrivial property $\P$, $\ind_\eta(\P)$, and hence $L_{\P}$, is denumerable, and so $\P$ is not obviously decidable. The boundary that demarcates decidable from undecidable properties is the content of \emph{Rice's theorem} \citeyear{Rice53}.

\begin{theorem}[Rice's Theorem]
\label{thm:RiceTheorem}
A \pc\ function property $\P$ is decidable iff $\P$ is trivial.
\end{theorem}

On first encounter Rice's theorem is surprising because it entails that decidable function properties are the exception rather than the rule. Take note, however, that Rice's theorem does not assert that no DTM can decide whether any \emph{particular} \pc\ function satisfies $\P$ or not, but rather that no DTM can decide whether \emph{every} \pc\ function satisfies $\P$ or not. This difference is important. It is the reason why, for example, we know many $\NP$-complete problems like $3$SAT or SUBSET-SUM despite the fact that $\P_{\NP}$ is nontrivial.

\section{Traits of Deterministic Turing Machines}
\label{sec:traitsofDTMs}

Not every DTM is the same. Whereas some DTMs compute different functions, some DTMs compute the same function but nevertheless compute differently. Again, there are a myriad of ways (infinitely many, in fact, by the padding lemma) to compute the Fibonacci numbers.

Whatever the distinction, it is often natural to prefer one set of DTMs over another. If, for example, time is a critical resource (as it is in medicine, autonomous driving, and stock trading), then it is natural to prefer that machine which ``gets the job done fastest'' over some other machine that merely ``gets the job done''. For more advanced systems, one may prefer one machine over another because it is ``more corrigible'' (see \sref{sec:applicationtoIMs}). Ultimately, whatever the reason for distinction, we seek a formal means for even talking about distinguishing DTMs. That is the purpose of this section, which is to be contrasted with the purpose of the last section, which sought a formal means for talking about distinguishing $\pc$ functions.

\subsection{Rice's Theorem for Traits}

Similar to how we distinguish \pc\ functions by the properties they satisfy, we shall distinguish DTMs by the \emph{traits} they possess.

\begin{definition}
A \emph{trait} $\T$ is a collection of DTMs. We say a DTM $M$ \emph{possesses} $\T$ iff $M \in \T$. Otherwise, $M \not\in \T$ and $M$ \emph{does not possess} $\T$.
\end{definition}

For example, the set comprised of all the universal DTMs is a trait. Moreover, it follows from \eref{eq:SIMsubDTM} that $\IM$ is a trait---it is by definition the collection of DTMs that ``possess intelligence''.

Analogous to properties of \pc\ functions, an indexing $\eta : \DTM \rightarrow \NN$ and a trait $\T$ generate a language over $\NN$:
\begin{equation}
L_\T \defeq \left\{\eta(M) \mid M \in \T\right\}.
\label{eq:traitlanguage}
\end{equation}
\begin{definition}
A trait $\T$ is \emph{decidable} iff $L_\T$ is decidable.
\end{definition}
This definition makes sense because $M \in \T$ iff $\eta(M) \in L_\T$.

Like properties of \pc\ functions, it is desirable to know what traits are decidable. The exact criterion requires Definitions~\ref{def:nontrivial} and \ref{def:semantic} below.
\begin{definition}
\label{def:nontrivial}
A trait $\T$ is \emph{trivial} iff $\T$ is $\emptyset$ or $\DTM$. Otherwise, $\T$ is \emph{nontrivial}.
\end{definition}

Due to similarity of terms, one might suspect that Rice's theorem applies in spades to traits. This thinking, however, is misguided: fundamentally, traits and properties of \pc\ functions are different mathematical objects, so the theorems concerning one do not immediately concern the other. In fact, it is false that every nontrivial trait is undecidable because the trait $\{M \mid \text{$M$ has $n$ states}\}$ is both nontrivial and decidable for every positive integer $n$. An analogue of Rice's theorem for traits, therefore, requires a finer characterization.

\begin{definition}
\label{def:semantic}
The \emph{machine class} $[M]$ of a DTM $M$ is the collection of DTMs that compute the same proper function as $M$. That is, $[M] \defeq \{N \in \DTM \mid \varphi_N \simeq \varphi_M\}.$ We say a trait $\T$ is \emph{semantic} iff $[M] \subseteq \T$ for every $M \in \T$. Otherwise, $\T$ is \emph{syntactic}.
\end{definition}

Intuitively, semantic traits distinguish DTMs based on what they do, not how they do it. Things like runtime, data, memory management, and so forth are all semantically incidental (and thus syntactic) because they do not change the input-output relationship of a DTM.

Evidently, every trivial trait is semantic but not every semantic trait is trivial. Moreover, it follows from the padding lemma that every nontrivial semantic trait is denumerable. It is therefore plausible that deciding nontrivial semantic traits is undecidable. Indeed, this is the case.
\begin{theorem}[Rice's Theorem for Traits]
\label{thm:ricetraits}
If $\T$ is a nontrivial and semantic trait, then $\T$ is undecidable.
\end{theorem}
\begin{proof}
Since $\T$ is semantic, $\eta(M) \in L_\T$ iff $\ind(\varphi_M) \subseteq L_\T$. Thus, $\bigcup_{M \in \T} \ind(\varphi_M) = L_\T.$ Now consider the \pc\ function property $\P_\T \defeq \{\varphi_M \mid M \in \T\}$. By \eref{eq:indicesofproperties},
\begin{equation}
\ind_\eta(\P_\T) = \bigcup_{\varphi \in \P_\T} \ind_\eta(\varphi) = \bigcup_{M \in \T} \ind_\eta(\varphi_M) = L_\T.
\end{equation}
Deciding $L_\T$, therefore, is equivalent to deciding the \pc\ function property $\P_\T$. By Rice's theorem, $\P_\T$ is decidable iff $\P_\T$ is trivial. But clearly $\P_\T$ is not empty and it does not contain all \pc\ functions. Therefore, $\P_\T$ is nontrivial and so $L_\T$ is undecidable.
\end{proof}

\thmref{thm:ricetraits} is evidently the natural analogue of Rice's theorem for traits. For this reason, we call it as such, and often refer to it by the abbreviation ``RTT''. Importantly, however, unlike Rice's theorem, the converse of RTT is false.
\begin{proposition}
\label{prop:converseisfalse}
The converse of \thmref{thm:ricetraits} is false.
\end{proposition}
\begin{proof}
By way of contradiction, it suffices to exhibit a nontrivial and syntactic trait that is undecidable. To this end, consider the trait $\T_{\halt}(n)$, where for each $n \in \NN$,
\begin{equation}
\T_{\halt}(n) \defeq \{M \in \DTM \mid \text{$M$ has $n$ states and $\forall \sigma \in \Sigma^* : \varphi_M(\sigma)\downarrow$}\}.
\end{equation}
It is plain that for every $n \in \NN$, $\T_\halt(n)$ is nontrivial. Moreover, for any $M \in \T_\halt(n)$, not every DTM in $[M]$ has $n$ states. Therefore, for every $n \in \NN$, $\T_\halt(n)$ is syntactic. It remains to prove that $\T_\halt(n)$ is undecidable. For contradiction, suppose that for all $n \in \NN$, $\T_\halt(n)$ is decidable. Then, there is a DTM $H$ such that for every $\sigma \in \Sigma^*$ and every $M \in \DTM$,\footnote{Here, in order for $H$ to take decimal- and $\Sigma$-string-valued inputs, its input alphabet is taken to be $\Sigma \cup \{0,1,2,3,4,5,6,7,8,9\}$. The same is also true for the DTM $S$, which is introduced later in the proof.}
\begin{equation}
\varphi_H(\eta(M), \sigma, n) = 
\begin{cases}
1 & \text{if $M$ has $n$ states and $\varphi_M(\sigma)\downarrow$},\\
0 & \text{otherwise}.
\end{cases}
\end{equation}
As $\{M \mid M\text{ has $n$ states}\}$ is decidable for every $n$, there exists a DTM $S$ such that for every $M \in \DTM$, $\varphi_S(\eta(M))$ equals the number of states of $M$. There is therefore a DTM $N$ such that for all $\sigma \in \Sigma^*$ and all $M \in \DTM$,
\begin{equation}
\varphi_N : (\eta(M), \sigma) \mapsto \varphi_H\left(\eta(M), \sigma, \varphi_S(\eta(M))\right).
\end{equation}
By the definitions of $N$ and $H$, it holds that $\varphi_N(\eta(M), \sigma)$ equals $1$ if $\varphi_M(\sigma)\downarrow$ and $M$ has a number of states equal to the number of states of $M$. Otherwise, $\varphi_N(\eta(M), \sigma)$ equals $0$. But ``$M$ has a number of states equal to the number of states of $M$'' is a tautology, so $\varphi_N(\eta(M), \sigma)$ equals $1$ if $\varphi_M(\sigma)\downarrow$, and it equals $0$ otherwise. $N$, therefore, decides the halting problem, which is impossible.
\end{proof}

Thus, in consequence to RTT and \propref{prop:converseisfalse}, we learn that it is necessary, but in no way generally sufficient, that a nontrivial trait be syntactic in order for it to be decidable. It is therefore desirable to have a finer characterization of which syntactic traits are decidable and which are not. 

\subsection{An Undecidability Criterion for Syntactic Traits}

The following lemma is trivial but has many useful corollaries.
\begin{lemma}
\label{lem:semantictraits}
Let $\SEM \subseteq 2^\DTM$ be the collection of all semantic traits. Then $\SEM$ forms an algebra over $\DTM$. That is,
\begin{enumerate}[(i)]
\item $\emptyset \in \SEM$,
\item $\T \in \SEM \implies \DTM \backslash \T \in \SEM$,
\item $\T_1, \T_2 \in \SEM \implies \T_1 \cup \T_2 \in \SEM$.
\end{enumerate}
\end{lemma}

If an arbitrary but finite logical combination $\overline{\T}$ of nontrivial semantic traits is itself nontrivial, then it follows straightforwardly from RTT and \lemref{lem:semantictraits} that $\overline{\T}$ is semantic and undecidable. More precisely:
\begin{proposition}
\label{cor:intersectiontraits}
Let $\T_1, \dots, \T_n$ be nontrivial semantic traits and let each $\bigcirc_i$ be either $\cup$ or $\cap$. If $\overline{\T} = \T_1 \bigcirc_1 \T_2 \bigcirc_2 \cdots \bigcirc_{n-1} \T_n$ is nontrivial, then $\overline{\T}$ is semantic and undecidable.
\end{proposition}

Consequently, to formally prove that any given DTM possesses a given trait $\T$ or not, it is necessary that $\T$ be neither semantic nor a finite, logical combination of semantic traits. In other words, to formally prove whether any given DTM possesses a desirable trait or not requires characterizing the desirable trait not strictly by what the DTM does, but also by how it does it. This point is paramount, and it motivates a distinction between the \emph{semantic} and \emph{syntactic parts} of a given trait.
\begin{definition}
The \emph{semantic part} of a trait $\T$ is the set $\sem(\T) \defeq \{M \mid [M] \subseteq \T\}$, while the \emph{syntactic part} of $\T$ is the set $\syn(\T) \defeq \T \backslash \sem(\T)$.
\end{definition}

Evidently, $\T = \sem(\T)$ iff $\T$ is semantic. Therefore, $\T$ is syntactic iff $\syn(\T) \neq \emptyset$. In general, it holds that $\T = \sem(\T) \cup \syn(\T)$ for every trait $\T$, which, together with \eref{eq:traitlanguage}, implies
\begin{equation}
L_\T = L_{\sem(\T)\, \cup\, \syn(\T)} = L_{\sem(\T)} \cup L_{\syn(\T)}.
\label{eq:semsynseparation}
\end{equation}
This equation is key, because we can leverage it along with the following lemma to establish an undecidability criterion for \emph{syntactic} traits.
\begin{lemma}
\label{lem:finitelemma}
Let $L_1$ and $L_2$ be languages over $\Sigma^*$ such that $|L_2|$ is finite. If $L_1 \cup L_2$ is decidable, then $L_1$ is decidable.
\end{lemma}
\begin{proof}
For all $\sigma \in \Sigma^*$,
\begin{equation}
\sigma \in L_1 \iff (\sigma \in L_1 \cap L_2) \lor (\sigma \in L_1 \cup L_2 \land \sigma \not\in L_2).
\end{equation}
Therefore, to decide $L_1$ it suffices to decide $L_2$, $L_1 \cap L_2$, and $L_1 \cup L_2$. But these are all decidable because $L_1 \cup L_2$ is decidable by hypothesis, and $L_2$ and $L_1 \cap L_2$ are finite sets.
\end{proof}

\begin{theorem}
\label{thm:mainthm}
If $\T$ is a trait with $\sem(\T) \neq \emptyset$ and $0 < |\syn(\T)| < \infty$, then $\T$ is undecidable.
\end{theorem}
\begin{proof}
Suppose for contradiction that $\T$ is decidable. Then, by \eref{eq:semsynseparation}, $L_{\sem(\T)} \cup L_{\syn(\T)}$ is a decidable language. With $L_1 = L_{\sem(\T)}$ and $L_2 = L_{\syn(\T)}$, we satisfy the premises of \lemref{lem:finitelemma}. Therefore, $L_{\sem(\T)}$ is a decidable language, so $\sem(\T)$ is a decidable trait. However, $\sem(\T) \neq \emptyset$ by assumption, and evidently $\sem(\T) \neq \DTM$ for otherwise $|\syn(\T)| = 0$. Therefore, $\sem(\T)$ is undecidable by RTT, which is a contradiction.
\end{proof}

Whereas before, in consequence to RTT and \propref{prop:converseisfalse}, we saw that it was necessary, but in no way generally sufficient, that a nontrivial trait $\T$ be syntactic in order for $\T$ to be decidable, we now see, in consequence to \thmref{thm:mainthm}, that we also require the syntactic part of $\T$ to be \emph{infinite} in order for $\T$ to be decidable (supposing, also, that $\T$ has a nonempty semantic part). In the next section, we explore one way by which to distinguish DTMs in a given machine class, which entails many traits whose syntactic parts are indeed infinite.

\subsection{Resource Bounded Traits}

In any computation, it often matters how efficient the computation is with respect to a particular resource. For example, it is often preferred to have not just a machine that computes a $\pc$ function $\varphi$, but rather one that computes $\varphi$ \emph{quickly}. Here, we introduce a rigorous framework for what it means to have a machine, and hence a trait, be ``resource bounded''. Our framework mostly derives from Manual Blum's axioms for abstract complexity measures. We then prove that so-called ``discriminating'' resource bounded traits are necessarily syntactic, which proves that the decidability of such traits cannot be determined by RTT alone. 

The following definition is adapted from \citeA{Blum67}.

\begin{definition}
A \emph{resource measure} $\Phi$ (also known as an \emph{abstract complexity measure}) is a $\pc$ function $\Phi : \NN \times \Sigma^* \rightarrow \NN$ such that for all $M \in \DTM$,
\begin{enumerate}[(i)]
\item for all $\sigma \in \Sigma^*$, $\Phi(\eta(M), \sigma) \downarrow$ iff $\varphi_M(\sigma)\downarrow$,
\item for all $n \in \NN$, the language $\left\{(\eta(M), \sigma, n) \mid \Phi(\eta(M), \sigma) = n\right\}$ is decidable.
\end{enumerate}
\end{definition}

Here, the first condition ensures that the resource measure is defined when and only when the computation whose resources one is measuring has completed. It makes no sense, for example, to measure the resources needed to do an undefined computation. On the other hand, the second condition ensures that measuring the resource measure is computable. In other words, it ensures that there is a computable fact of the matter of just how much of the resource $\Phi$ is actually used on any run of the computation.

Among the many possible resource measures, two of the most important are $\TIME$ and $\SPACE$, which satisfy, for all $M \in \DTM$ and all $\sigma \in \Sigma^*$,
\begin{align}
\TIME(\eta(M), \sigma) &= \text{the time it takes $M$ to halt on input $\sigma$,}\\
\SPACE(\eta(M), \sigma) &= \text{the space $M$ uses on input $\sigma$ prior to halting.}
\end{align}

Our main use of any given resource measure $\Phi$ will be to discriminate between those DTMs that require ``a large amount of $\Phi$'' and those that do not. Toward such distinguishability, the following definition is paramount.
\begin{definition}
Let $\T$ be a trait and let $\Phi$ be a resource measure. We say $\T$ is \emph{$\Phi$-bounded} iff there exists a total $\pc$ function $\xi: \NN \rightarrow \NN$ such that $\T \subseteq \T_{\Phi}(\xi)$, where
\begin{equation}
\T_\Phi(\xi) \defeq \left\{M \in \DTM \mid \forall \sigma \in \Sigma^* : \Phi(\eta(M), \sigma) \leq \xi(|\sigma|)\right\}
\end{equation}
and $|\sigma|$ is the length of $\sigma$.
\end{definition}

For example, a trait $\T$ is \emph{$\TIME$-bounded} iff there is a total $\pc$ function $t : \NN \rightarrow \NN$ such that 
\begin{equation}
\T \subseteq \left\{M \in \DTM \mid \forall \sigma \in \Sigma^* : \text{$M(\sigma)$ halts in time at most $t(|\sigma|)$}\right\}.
\end{equation}
Similarly, $\T$ is \emph{$\SPACE$-bounded} iff there is a total $\pc$ function $s : \NN \rightarrow \NN$ such that 
\begin{equation}
\T \subseteq \left\{M \in \DTM \mid \forall \sigma \in \Sigma^* : \text{$M(\sigma)$ halts in space at most $s(|\sigma|)$}\right\}.
\end{equation}
Incidentally, since a DTM that halts in time at most $t$ can use at most $t$ space, any $\TIME$-bounded trait is also $\SPACE$-bounded. Conversely, if $\T$ is $\SPACE$-bounded by some function $s$, then $\T$ is $\TIME$-bounded by $|\Gamma|^s$, where $\Gamma$ is the tape alphabet of the DTM in question. Therefore, $\T$ is $\TIME$-bounded iff $\T$ is $\SPACE$-bounded.

\begin{definition}
We say a resource measure $\Phi : \NN \times \Sigma^* \rightarrow \NN$ is \emph{discriminating} iff for all $M \in \DTM$, there exists $N \in [M]$ such that for all $\sigma \in \Sigma^*$, $\Phi(\eta(N), \sigma) > \Phi(\eta(M), \sigma)$.
\end{definition}

A discriminating resource measure $\Phi$ thereby affords a formal means for preferring ``this machine over that machine'' for any pair of machines, in the sense that while the two machines may ultimately compute the same function (and hence belong to the same machine class), only one of them will actually compute the function in a way that moderates its use of the resource $\Phi$. That is why this notion is useful.

It is plain that both $\TIME$ and $\SPACE$ are discriminating resource measures. This follows because given any $M \in \DTM$, one can always contrive another DTM $N \in [M]$ such that for all $\sigma \in \Sigma^*$, $\TIME(\eta(N), \sigma) > \TIME(\eta(M), \sigma)$ and $\SPACE(\eta(N), \sigma) > \SPACE(\eta(M), \sigma)$ by, for example, adding unnecessary loops and function calls in the algorithm of $N$ that take up time and space.

Our main use of $\Phi$-bounded traits stems from the next proposition.
\begin{proposition}
If $\T$ is a nontrivial trait that is $\Phi$-bounded by a discriminating resource measure $\Phi$, then $\T$ is syntactic.
\label{prop:resourceprop}
\end{proposition}
\begin{proof}
Suppose for contradiction that $\T$ is semantic. Since $\T$ is nontrivial, there exists $M \in \T$. Since $\T$ is semantic and $\Phi$-bounded, $[M] \subseteq \T \subseteq \T_\Phi(\xi)$. In particular, for all $N \in [M]$ and for all $\sigma \in \Sigma^*$, $\Phi(\eta(N), \sigma) \leq \xi(|\sigma|)$. But $\Phi$ is a discriminating resource measure, so it holds for all $N \in [M]$ that there is $O_1 \in [M]$ such that for all $\sigma \in \Sigma^*$, $\Phi(\eta(O_1), \sigma) > \Phi(\eta(N), \sigma)$. By induction and the fact that $[M]$ is countably infinite, it holds more generally that for all $n \in \NN$, there are $O_1, \dots, O_n \in [M]$ such that for all $\sigma \in \Sigma^*$,
\begin{equation}
\Phi(\eta(O_1), \sigma) < \Phi(\eta(O_2), \sigma) < \dots < \Phi(\eta(O_n), \sigma) \leq \xi(|\sigma|).
\label{eq:inequality}
\end{equation}
But $\Phi(\eta(O_n), \sigma) \rightarrow \infty$ as $n \rightarrow \infty$, which contradicts \eref{eq:inequality}. Therefore, $\T$ is syntactic.
\end{proof}

Consequently, any nontrivial trait $\T$ that is $\Phi$-bounded by a discriminating resource measure $\Phi$ does not satisfy the premises of RTT. As a result, for any such $\T$, RTT cannot be used to determine whether $\T$ is decidable or not. Rather, a theorem that applies to syntactic traits is needed, such as our \thmref{thm:mainthm}. We shall appeal to this conclusion a number of times in the coming section.

\section{Traits of Intelligent Machines}
\label{sec:applicationtoIMs}

In this section we apply the formalism of the previous two sections to the Machine Intelligence Hypothesis (MIH) and explore the computability consequences.

\subsection{Intelligent Machines and the Machine Intelligence Hypothesis}

As outlined in \sref{sec:MIH}, the MIH encapsulates the assumptions necessary to justify the inclusion
\begin{equation}
\IM \subseteq \DTM,
\label{eq:MIH2}
\end{equation}
where $\IM$ is the set of DTMs that possess intelligence, according to some agreeable (and sufficiently formal) definition of ``intelligence''. By the formalism of the previous section, \eref{eq:MIH2} entails that $\IM$ is a trait of DTMs. We now explore in turn the questions of whether $\IM$ is trivial or nontrivial, semantic or syntactic.

Is $\IM$ trivial? If so, then either $\IM = \emptyset$ or $\IM = \DTM$. If, on one hand, $\IM = \DTM$, then every DTM possesses intelligence. This, of course, is quite presumptuous, at least because regarding a garden-variety Linux machine as intelligent will invariably raise objection, but also because formal judgements about intelligence almost surely involve spatiotemporal factors \cite{Fry96,Haier16,Chollet19}. For example, if, in Knuth's up-arrow notation, a DTM takes $\Omega(2 \uparrow \uparrow \uparrow \uparrow n)$ space or time to perform elementary arithmetical tasks like adding or multiplying two $n$-bit numbers, then it is plausibly not intelligent. Of course, even an efficient adder and multiplier like a pocket calculator is also plausibly not intelligent, but for entirely different reasons. This is all to say that $\IM$ is plausibly a $\TIME$-bounded trait---a consideration we shall rejoin later---and hence that $\IM \neq \DTM$. Separately, \citeA{Bostrom14} has argued that every $M \in \IM$ is a universal DTM, which also entails $\IM \neq \DTM$ because not every DTM is universal.

If, on the other hand, $\IM = \emptyset$, then DTMs do not possess intelligence. This may well be plausible, particularly if, for example, one regards intelligence as a trait unique to biological species, like dolphins and human beings. However, such a ``carbon chauvinistic'' attitude opposes the thesis of substrate independence, which the author believes is well-argued for \cite{Tegmark17}. Cooperating silicon-based GPUs, therefore, should have no trouble exhibiting a sort of intelligence, provided they are programmed to do so. Ultimately, of course, whether $\IM = \emptyset$ or not is a function of one's definition of ``intelligence''. For the sake of argument, we shall proceed assuming that ``intelligence'' is defined in a restrictive but non-carbon-chauvinistic way so that some, but not all, DTMs possess intelligence.

Is $\IM$ semantic? This question is more complicated. It is tantamount to asking if all that should matter in the determination of a system's intelligence is its input-output relationship. Of course, this question is far from new, as it has been the subject of a multi-century debate in various spheres of philosophy. Is it the case, for example, that a program that fluently converses in Chinese also \emph{understands} Chinese, \`a la John Searle's \citeyear{Searle80} famous Chinese room argument? Or is it necessary that a more functionalist, and not purely behavioralist, analysis is necessary to conclude anything at all about the underlying intelligence of the program, such as whether or not it ``understands'' Chinese?

As mentioned above, it appears (at least to the author) that DTMs that take, say, doubly-exponential time to compute a fundamentally constant-time operation should not be dubbed ``intelligent''. Imagine a world, for instance, in which ChatGPT was publicly released, but was such that it took a whole year to reply to any query it was fed. The inaugural responses would have only recently come in,\footnote{OpenAI publicly released ChatGPT on November 30, 2022.} and while the responses may be impressive, that they took so long to arrive would plausibly influence our perception of the inherent ``intelligence'' of ChatGPT. This is to say that not only does runtime seem a relevant factor in determining whether $\IM$ is trivial or nontrivial, but it also seems a relevant factor in determining whether $\IM$ is semantic or syntactic. In particular, intelligence, in our opinion, is correlated with efficiency, and so any DTM that ``gets the job done fast'' will undoubtedly be regarded as more intelligent than any DTM that merely ``gets the job done.'' We formally encapsulate this hypothesis in the next conjecture.
\begin{conjecture}
$\IM$ is $\TIME$-bounded.
\label{conj:imtime}
\end{conjecture}

Assuming \conjref{conj:imtime}, it then follows from \propref{prop:resourceprop}, as well as the fact that $\TIME$ is a \emph{discriminating} resource measure, that $\IM$ is syntactic. To conclude, then, on the basis of RTT alone that $\IM$ is undecidable is a logical error, because the premises of RTT do not apply. In particular, we find the conclusion reached by \citeA{Alfonseca21} mistaken, because $\IM$, as we suspect here, is not a semantic trait, as it was assumed there, but rather a syntactic one. Instead, on account of our \thmref{thm:mainthm} and our suspicion that $\IM$ is syntactic, to establish the undecidability of $\IM$, it is sufficient to show that $\sem(\IM) \neq \emptyset$ (i.e., that intelligence is congenitally characterized by the input-output behavior of a machine) and that $0 < |\syn(\IM)| < \infty$ (i.e., that the syntactic part that characterize an intelligent machine is restrictive in the sense that only finitely many machines possess it). This is not done by \citeA{Alfonseca21}, and so we find their conclusion that $\IM$ is undecidable to be flawed. Indeed, until further deliberations lead to a more complete and formal definition of ``intelligence'', it remains unproven if determining whether an arbitrary machine possesses intelligence is decidable or not.

\subsection{Desired Traits of Intelligent Machines}
\label{sec:othertraitsofIM}

If AI research is to succeed in its goal of fabricating machines that exhibit intelligent behavior \cite{Russell17,Russell20}, then it is important to consider what traits in addition to intelligence those machines ought to possess. This is so, at least in part, because the ubiquity of AI in technology today suggests that future intelligent machines will almost surely interact with and influence us humans. Intelligence, then, is not the only desired trait, because intelligence by itself can wreak both good and bad for humanity.

Let us suppose, then, that we have a (partial) solution to the easy problem of desires, say, in the form of a long list. No matter how we got this solution,\footnote{Perhaps by some campaign run by the United Nations in which every living human is thoroughly interview, or perhaps, more reasonably, by some advanced machine that learned our desires by observation like in inverse reinforcement learning \cite{Russell19}.} the relevant thing is that the list, however long, is necessarily finite, because there is no capacity to store an infinite amount of information in the Milky Way, let alone on this planet. Therefore, the list entails finitely many machine traits $\T_1, \T_2, \dots, \T_n$, corresponding to the desires we have of the intelligent machines, in addition to them just merely possessing intelligence. Of all possible intelligent machines, then, we seek to create only those ``desired intelligent machines'' that are contained in the set $\DIM$, which is formally given by
\begin{equation}
\DIM = \IM \cap \T_1 \bigcirc_1 \T_2 \bigcirc _2 \cdots \bigcirc_{n-1} \T_n.
\label{eq:DIM}
\end{equation}
Here, as in \propref{cor:intersectiontraits}, each $\bigcirc_i$ is either $\cup$ or $\cap$. Of course, each $\bigcirc_i$ is probably $\cap$, as our list of desires is likely just a corpus of AND statements. However, we do not rule out the possibility of there being pairs of traits that are, for example, logically incompatible, in which case the best we can logically do is OR them.

In \eref{eq:DIM}, the first $\cap$ is there to enforce that $\DIM \subseteq \IM$. This way, no matter what the intelligent machines we desire ultimately are, they are at the very least intelligent. If it is the case, however, that intelligence is not all we desire of intelligent machines, then the containment of $\DIM$ in $\IM$ is strict, which is to say $\IM \cap \T_1 \bigcirc_1 \T_2 \bigcirc _2 \cdots \bigcirc_{n-1} \T_n \subsetneq \IM$. As this condition does \emph{not} impose that each desired trait $\T_j$ be contained in $\IM$, the traits we desire of intelligent machines need not be traits that exclusively intelligent machines possess. Nevertheless, by the way the initial $\cap$ distributes over each $\bigcirc_i$ in \eref{eq:DIM}, in order for $\DIM$ to be nonempty, there must be some desired trait $\T_j$ that some intelligent machine possesses. This much is self-evident.

Now, by our preceding argument that $\IM$ is plausibly nontrivial, it holds that if $\DIM$ is itself nonempty, then $\DIM$ is plausibly nontrivial, too. What justifies, then, that $\DIM$ is nonempty? Unfortunately, we know of nothing formal, but rest that assumption on a belief that it is possible to have intelligent machines that fulfill our (partial) solution to the easy problem of desires. This may well be wrong, in which case any intelligent machine that we do eventually build is destined to at the very least disappoint humanity in some basic capacity. However, in the spirit of optimism and also to analyze the most interesting case, we shall proceed assuming that $\DIM \neq \emptyset$.

If $\IM$ is semantic, as are all $\T_1, \T_2, \dots, \T_n$, then $\DIM$ is semantic (\propref{cor:intersectiontraits}), in which case $\DIM$ is formally undecidable. However, if at least one of $\IM, \T_1, \T_2, \dots, \T_n$ is syntactic, then $\DIM$ \emph{may} be syntactic, in which case $\DIM$ \emph{may} be formally decidable (albeit not necessarily, even if $\DIM$ is syntactic, as both \propref{prop:converseisfalse} and \thmref{thm:mainthm} prove).\footnote{Incidentally, it is straightforward to prove that if every $\bigcirc_i$ in \eref{eq:DIM} is $\cap$ and if at least one of $\IM, \T_1, \T_2, \dots, \T_n$ is syntactic, then $\DIM$ is necessarily syntactic.} To get a better handle on the decidability of $\DIM$, therefore, it is fruitful to consider a handful of traits $\T_j$ that we humans may wish to constitute $\DIM$, and to then discuss, in turn, whether each is likely to be trivial or nontrivial, semantic or syntactic. For brevity, we shall only consider three such traits here, but of course the list could be extended considerably.

\subsubsection{Contained Machines: $\CM$}

Inspired by \citeS{Lampson73} ``confinement problem'', as well as \citeS{Yampolskiy12} generalization to AI (the ``AI confinement problem''), let a \emph{contained machine} $M$ be defined as any DTM whose proper function $\varphi_M : \Sigma^* \rightarrow \Sigma^*$ is, one, computed by $M$ in such a way so to neither encode nor entail the leakage of classified information and, two, is such that for every $\sigma \in \dom(\varphi_M)$, the output $\varphi_M(\sigma)$ itself neither encodes nor entails the leakage of classified information. Here, by ``classified information'' we do not necessarily mean ``classified'' in the sense of material that, say, the United States government has deemed ``sensitive'', such as the nuclear launch codes, but, rather, any information that we humans do not want other (potentially superintelligent) agents to have free access to because such access would negatively impact humanity. In other words, ``classified information'' constitutes that set $L_\class$ of all strings $\sigma \in \Sigma^*$ that entail what \citeA{Bostrom11} calls ``information hazards'', which are ``[risks that arise] from the dissemination or the potential dissemination of (true) information that may cause harm or enable some agent to cause harm''.

With this in mind, the first condition in our definition of a contained machine $M$ ensures that $M$, for example, does not in the process of computing $\varphi_M$ reveal classified information, say, by encoding it in binary via the left and right steps $M$ takes as it traverses its tape, while the second condition ensures that the output of $M$ does not entail, for example, a later action by any agent that could compromise the security of the classified information. Both of these conditions ensure that no information hazards can arise from $M$.

In the context of AI safety, the aspiration is to create ``contained intelligent machines'', which are essentially intelligent machines that can be ``sealed'' so as to not inflict harm on humankind \cite{Yampolskiy12}. Along these lines, \citeA{Chalmers10} proposed that for safety reasons, advanced AI systems should first be ``restricted to simulated virtual worlds until their behavioral tendencies could be fully understood under the controlled conditions''. However, while it is reasonable to want intelligent machines to be somewhat contained, it is not reasonable to want intelligent machines to be completely jailed away in a virtual world, particularly if we want anything of practical value from the machines themselves. Indeed, as \citeA{Chalmers10} also points out, complete imprisonment in some virtual world entails that all information about the ``real world'' is classified, so no completely contained machine can help with any real world problem. Thus, a completely contained machine is pointless from a practical point of view, so it is necessary that a ``useful'' machine use some amount of ``real world'' information in its processing. 

Let us suppose, then, that all information is designated ``classified'' and ``unclassified'', given formally by the languages $L_{\class}$ and $L_{\unclass}$, respectively. Both of these sets are subsets of $\Sigma^*$ and they evidently satisfy $L_{\class} \cap L_{\unclass} = \emptyset$. Moreover, we want the machines to act in the real world, so $L_{\unclass}$ contains strings that are relevant to the real world, such as the content in an introductory physics textbook. Given a designation like this of the classified and unclassified information, our definition of a contained machine defines the trait of all contained machines:
\begin{equation}
\CM \defeq \left\{M \in \DTM \mid \text{$M$ is contained}\right\}.
\label{eq:CM}
\end{equation}
Note that by the second part of the definition of a contained machine $M$, it is necessary that $M$ satisfy $\varphi_M : \Sigma^* \rightarrow L_{\unclass}$. For, if there is $\sigma \in \dom(\varphi_M)$ for which $\varphi_M(\sigma) \in L_\class$, then the output of $M$ contains classified information, which means that $M$ is not contained.\footnote{Ideally, it is desirable to also enforce a condition like there being no $\pc$ function $f : L_\unclass \rightarrow L_\class$. This way, it would not be possible for a contained machine to encode classified information with unclassified information. But this, of course, is too general a condition to practically enforce.} In consequence, it is necessary that
\begin{equation}
\CM \subseteq \{M \in \DTM \mid \varphi_M : \Sigma^* \rightarrow L_{\unclass}\}.
\label{eq:CMupperbound}
\end{equation}

Toward addressing the question of whether $\CM$ is decidable (given a particular demarcation of the classified and unclassified information), we now turn to the questions of whether $\CM$ is trivial or nontrivial, semantic or syntactic. 

Given the upper bound \neref{eq:CMupperbound} on $\CM$, it is straightforward to see that $\CM$ is nontrivial. Indeed, since $L_{\class} \cap L_{\unclass} = \emptyset$, it holds that $L_\unclass \neq \Sigma^*$. Therefore, $\CM \neq \DTM$, for not all DTMs have a proper function that map $\Sigma^*$ to $L_\unclass$. Additionally, the machine $M$ with $\dom(\varphi_M) = L_\unclass$ and that, for every input $\sigma \in L_\unclass$, immediately halts and returns $\sigma$, is manifestly contained. Therefore, $\CM \neq \emptyset$, and so $\CM$ is nontrivial. We now address the question of whether $\CM$ is semantic or syntactic. 

In their paper, \citeA{Alfonseca21} take $L_\class$ to be those strings $\sigma \in \Sigma^*$ that ``harm humans'', where a sufficiently precise notion of what it means to ``harm humans'' is assumed. They then regard $\CM$ as the set of those machines that ``do not harm humans'', which is to say that by their definition, $\CM$ \emph{equals} $\{M \in \DTM \mid \varphi_M : \Sigma^* \rightarrow L_{\unclass}\}$ in \eref{eq:CMupperbound}. If that is right, then, as they correctly point out, $\CM$ is semantic (and hence undecidable by RTT) because the condition $\varphi_M : \Sigma^* \rightarrow L_{\unclass}$ is satisfied by all $N \in [M]$. 

However, we suspect that $\CM$ is syntactic, no matter one's choice of $L_\class$ and $L_\unclass$ (provided they are ``nontrivial'' in the sense that both $L_\class$ and $L_\unclass$ are nonempty and together satisfy $L_\class \cap L_\unclass = \emptyset$). If we are right, then our conclusions suggest that it remains open whether $\CM$ is decidable or not, contrary to \citeA{Alfonseca21}.

Our belief that $\CM$ is syntactic ties to our inclusion of the first condition in our definition of a contained machine $M$---namely, that $M$ must compute its proper function $\varphi_M$ in such a way so to neither encode nor entail the leakage of classified information. This in and of itself is a syntactic constraint that not all machines in $[M]$ satisfy. There is, for example, a DTM $N \in [M]$ that on every input $\sigma \in \Sigma^*$ first traverses to the right of $\sigma$ on its tape, writes down a classified string $\chi \in L_\class$, erases it, and then simulates $M$ on the original input $\sigma$. To any observer watching the tape of $N$ as $N$ operates, or to any universal machine that can simulate $N$ given $\eta(N)$, it is possible to obtain the classified string $\chi$. In this way, $N$ is capable of leaking $\chi$, despite never deliberately outputting $\chi$. More generally, any machine $O \in [M]$ that can temporarily store classified information in memory can unwittingly leak such information to any observer (or to any universal machine) with access to its memory bank, despite never actually outputting the classified information itself. Thus, we suspect that the first condition in our definition of a contained machine is necessary to ensure the security of classified information. In consequence, $\CM$ is syntactic, and so contrary to what \citeA{Alfonseca21} find, it does not follow from RTT that $\CM$ is undecidable.

Nevertheless, even if we are right and $\CM$ is syntactic, there remain several experimental reasons to expect $\CM$ to be undecidable. For example, an intelligent contained machine $M$---i.e., a DTM $M \in \IM \cap \CM$---that has access to ``real world'' information and that is able to interact directly with humans can, by interacting with unclassified information alone, conceivably convince a human to reprogram $M$ so to no longer be contained. This could be done, for example, by exploiting human psychology in some elaborate social engineering scheme \cite{Yampolskiy12}. Along these lines, \citeA{Yudkowsky02} demonstrated in a series of five experiments that individuals playing the role of ``gatekeeper'' could be outfoxed by Yudkowsky himself through nothing more than a chat interface. Indeed, in three of the five experiments, Yudkowsky ``escaped''. Similar results were replicated on a larger scale by \citeA{Corwin02}, which only increase our credence that $\CM$, even if syntactic, is undecidable.

\subsubsection{Moral Machines: $\MM$}

Similar to the definition of a contained machine, let a \emph{moral machine} $M$ be defined as any DTM whose proper function $\varphi_M : \Sigma^* \rightarrow \Sigma^*$ is, one, computed by $M$ in a morally responsible way and, two, is such that for every $\sigma \in \dom(\varphi_M)$, $\varphi_M(\sigma)$ entails a morally responsible action. Here, the first condition ensures that $M$, for example, does not require computational resources that would devastate humanity (such as an inordinate amount of energy or a data center the size of the earth), while the second condition ensures that the output of $M$ encodes a morally responsible action, so that any agent acting solely on the basis of the output of $M$ is a moral agent.

Now, while both of these conditions may appear reasonable for a machine to be moral, we do not for a moment claim that this is the ``right'' definition for a universal machine ethics. After all, what does it even mean to be ``morally responsible'' and what really is a ``moral agent''? If, as in \citeS{Asimov50} three laws of robotics, these terms reduce to a utilitarian-like maxim such as ``minimize harm to humans'', then the question pivots to determining what counts as ``harm''. But, as moral philosopher Derek Leben \citeyear{Leben18} argues in his book \emph{Ethics for Robots}, this question is annoyingly ambiguous:
\begin{quote}
Is it harming someone to insult them, lie to them, or trespass on their property? What about actions that violate a person's consent or dignity? Some actions are only likely to cause harm, but what's the threshold for \emph{likely} harm? Every action could \emph{possibly} lead to harm, so failing to specify this threshold will leave [machines] paralyzed with fear, unable to perform even the most basic of tasks.
\end{quote}
Additionally, in situations that are bona fide moral dilemmas, such as a future machine doctor deciding whether to respect a patient's wishes or to do what is best for their health, how should the machine act? Trolly problems like this are such that every action, including inaction, entail a serious iniquity.

We acknowledge, of course, that what should constitute a universal machine ethics is not a new question. Nevertheless, like in the definitions of contained machines above and superintelligent machines below, we can largely circumvent the centuries-old issue of what ``being moral'' actually means because, we think, the structure of our definition of what it means to be a ``moral machine'' is sufficiently close to the structure of the ``right'' definition to do justice to the task we care about---namely, to determine whether the set
\begin{equation}
\MM \defeq \left\{M \in \DTM \mid \text{$M$ is moral}\right\}
\label{eq:MM}
\end{equation}
of all moral machines is trivial or nontrivial, semantic or syntactic.

To this end, let us first address the ostensibly easy question of whether $\MM$ is even nonempty. As \citeA{Leben18} points out, given that machines---particularly those running contemporary ML algorithms---can emulate, and even trump, human performance in a myriad of tasks like image recognition, language use, video games, and driving, why should they not be able to emulate, and perhaps also trump, our moral judgements as well?

A fundamental obstacle, however, is the profound philosophical question of just how tied morality is to our humanity. If, as advocated by moral anti-realists, morality manifests from our own human cognition, then of course no non-human machine can exercise, under its own mechanistic volition, a unique moral judgement. Rather, all it could really do is regurgitate those imperfect, inconsistent, and bias judgements that it learned from watching us humans \cite{Leben18}. Moreover, there would be no ``right'' answer to any given moral quandary, and hence there would always be a subset of humans relative to which the machine answered immorally. It seems, then, that if each human had a say in what $\MM$ looked like, the overall set would turn out empty due to starkly incompatible moral views. This is the view the ethicists Wendell Wallach and Colin Allen extend in their book \emph{Moral Machines: Teaching Robots Right from Wrong} \citeyear{Wallach10}:
\begin{quote}
Given the range of perspectives regarding the morality of specific values, behaviors, and lifestyles, perhaps there is no single answer to the question of whose morality or what morality should be implemented in AI. Just as people have different moral standards, there is no reason why all computational systems must conform to the same code of behavior.
\end{quote}

However, as argued by moral-realists such as \citeA{Leben18}, if there is in any practical moral matter an objective and mind-independent set of solutions (which thereby debar the incongruous moral rationalizations that inform many of us humans \cite{Haidt01}), then the question of how the machine ought to behave reduces to one of cooperation. In particular, \citeA{Leben18} argues that in addition to internal consistency (which any intelligible and historically important moral theory like utilitarianism, libertarianism, Kantian ethics, or virtue ethics certainly is), moral theories should also be evaluated by how effectively they promote cooperative behavior. If this is correct, then the ``right'' moral theory to encode into our mechanistic progeny is not some deontological list of what is right versus what is wrong, nor is it a learning algorithm that ostensibly divines right versus wrong, but rather it is a function that promotes cooperation between all the parties involved in any given morally consequential affair. Therefore, game-theoretic algorithms for finding Pareto optimal solutions to multi-objective optimization problems (such as variants of \citeS{Benson98} algorithm for multi-objective linear programs) may be close siblings to the algorithms that underlie moral machines. Whatever the actual algorithms are, though, we take Leben's approach as evidence that $\MM \neq \emptyset$.\footnote{This reveals many of the author's own normative assumptions about morality. It is important for the reader to introspect on their own and to see where they land.}

Now, to finish addressing whether $\MM$ is trivial or nontrivial, it remains to decide if $\MM = \DTM$ or not. This, however, is easy to figure in the context of the above discussion, because we can just contrive a machine that always does or approximates the objectively worst moral thing. In the context of contractarianism, which is the theory to which Leben subscribes, this is possible by programming the machine to not cooperate with humans. An alternative, and perhaps more convincing argument that $\MM \neq \DTM$, follows from the existence of online recommendation algorithms, which are arguably destructive for humanity \cite{Allcott20}. In either case, we conclude that $\MM \neq \DTM$ and hence that $\MM$ is nontrivial.

Let us now turn to the remaining question: is $\MM$ semantic or syntactic? While the argument we used to establish $\MM \neq \emptyset$ may suggest that any machine $M \in \MM$ that implements a function that promotes cooperation between all the parties involved in a morally consequential affair is thereby moral, which would thus suggest that $\MM$ is semantic, this view discounts entirely the fact that in many morally consequential dilemmas, the time it takes to come to a decision is vital. Consider, for example, an autonomous bus carrying twenty schoolchildren that finds itself in a situation in which, one, the time to choose one of the following trajectories must be made in less than one second and, two, the only physically possible trajectories of the bus are either to continue straight off a 400 meter bluff, to stop by colliding with a singular bicyclist at 30 km/h, or to stop by colliding head-on with a car full of seniors at the equivalent of 65 km/h. In a circumstance like this, any algorithm that deliberates for too long will unwittingly lead to what is, in the author's opinion, the worst of the three outcomes.

Of course, this is one scenario, and a human in the driver's seat of that bus will not be dubbed ``immoral'' simply because they did not act swiftly enough, so why should the machine be? The point we wish to illustrate, however, is that if the algorithm is such that it deliberates for too long in almost every morally consequential affair (the bus scenario being but one instance), so much so that its action is almost always inaction, then under no circumstances should that machine be dubbed ``moral'' because in almost every scenario it engenders a morally unscrupulous outcome. Consequently, for a machine to be moral, it is necessary that it be somewhat computationally efficient. Incidentally, this is in part why our definition of a moral machine $M$ insists that $\varphi_M$ be computed in a ``morally responsible way''. Indeed, if $M$ is slow to compute $\varphi_M$, then, as we have just outlined, $M$ may unintentionally entail what is plausibly a morally irresponsible action. We therefore conjecture the following:
\begin{conjecture}
$\MM$ is $\TIME$-bounded.
\label{conj:mmtime}
\end{conjecture}

If this is right, then it follows from \propref{prop:resourceprop}, as well as the fact that $\TIME$ is a discriminating resource measure, that $\MM$ is syntactic. Consequently, $\MM$ does \emph{not} satisfy the premises of RTT, and hence it is \emph{not} necessarily the case that $\MM$ is undecidable. As with $\IM$ and $\CM$, the final say in the matter ultimately depends on more precisely defining the terms that makeup the definition of $\MM$. Perhaps then one can establish that $\sem(\MM) \neq \emptyset$ and $0 < |\syn(\MM)| < \infty$, in which case $\MM$ would be provably undecidable by our \thmref{thm:mainthm}.

At what point will a more formal definition of $\MM$ be known? As mentioned before, such a definition involves settling century-old issues like whether moral realism is right or what being ``morally responsible'' means. Nowadays, as evermore powerful machines are increasingly unleashed, it is crucial to properly address these philosophical questions. It is, in \citeS{Bostrom14} words, ``philosophy with a deadline''.

\subsubsection{Superintelligent Machines: $\SIM$}

Following \citeA{Good66}, ``let [a superintelligent] machine be defined as a machine that can far surpass all the intellectual activities of any man however clever.'' In other words, a DTM $M$ is \emph{superintelligent} iff in every intellectual domain, the intelligence of $M$ far surpasses the intelligence of every human. By this definition, it is plain that the set
\begin{equation}
\SIM \defeq \left\{M \in \DTM \mid \text{$M$ is superintelligent}\right\}
\label{eq:SIM}
\end{equation}
of all superintelligent machines is a subset of $\IM$. Supposing as before that $\IM \neq \DTM$, it follows trivially that $\SIM \neq \DTM$ as well. Whether $\SIM$ is trivial or not therefore rests on whether it is possible for a machine to be superintelligent at all.

By the definition of a superintelligent machine $M$, it is necessary that $M$ also exhibit \emph{general intelligence}, which means that $M$ possesses intelligence in every intellectual domain. Thus, the question of whether a superintelligent machine is possible is closely entangled with the question of whether an artificial general intelligence (AGI) is possible. 

While today there exist a myriad of programs that are ``superhuman'' in a variety of narrow intellectual domains (e.g. chess), it remains to fabricate a truly generally intelligent machine whose intellectual abilities span the intelligence spectrum. Arguably the closest we have come to this goal is OpenAI's GPT-4 \cite{Achiam23}, in the epochal wake of which speculations were raised by Microsoft researchers that GPT-4 possesses ``sparks'' of the first generally intelligent algorithm \cite{Bubeck23}. Whether or not this serves as substantive evidence that an AGI is possible is unclear. However, this does illustrate that vast technological leaps toward developing AGI remain possible, even if physical limitations on computing power are met. Absent any deeper principled reason why some other insuperable obstacle may obstruct the path to developing an AGI, it is only natural to presume that there exist algorithms that exhibit true general intelligence.

Assuming this, we now turn to the remaining question of whether $\SIM$ is nonempty, that is, whether it is possible for a generally intelligent machine to far surpass the intellectual abilities of every human in every intellectual domain. As expounded in \citeS{Bostrom14} comprehensive monograph \emph{Superintelligence}, the answer, it seems, is ``almost surely''. 

To see why, it is enlightening to consider the likelihood of a superintelligence being impossible on the assumption that an AGI as possible. Such a position effectively supposes that the general intelligence of machines is approximately upper-bounded by the general intelligence of humans. This could sensibly happen, for instance, if for every narrow intellectual activity in which machines dominate humans (e.g. chess), there is an activity in which humans dominate machines (e.g. perception, in the sense of \citeS{Moravec88} paradox). In that case, the general intelligence of humans and machines would, on the whole, level out. But just how tenable is this view?

The first thing to note is that a subscriber to this line of thinking will at the very least have greater and greater trouble defending themselves should machines continue to exceed human intelligence in more and more intellectual domains. The second thing to note, which we find much more problematic, is that this view resonates with a very parochial view of just how intelligent a physical system can become---a view which places the likes of Einstein on the far right of the ``intelligence scale'' and the likes of a mouse, say, on the far left. 

In defense of that scale, one might argue that the encroaching physical limitations of computing power will soon obstruct us from optimizing our algorithms any further. If that happens, and if our most optimized algorithms fail to exceed human intelligence in every intellectual domain, then no machine can exceed the general intelligence of humans, Einstein included. However, this argument misses the basic fact that there is more to growing an intelligent algorithm than by merely throwing inordinate amounts of compute at it.\footnote{Interestingly, viewed as a sort of biological computer, the human brain is capable of roughly an exaFLOP ($10^{18}$ operations per second) \cite{Madhavan23}, which puts our computing power on a par with Frontier's, the most powerful supercomputer in existence as of 2023. Therefore, if compute is all that matters in the exhibition of general intelligence, then Frontier is theoretically capable of whole brain emulation.} Indeed, as \citeA{Bostrom14} expounds, the rate of change in the intelligence of a system is, as indicated above, proportional to the optimization power being applied to the system, but also it is proportional to just how responsive a system is to a given amount of such optimization power. Calling the inverse of responsiveness ``recalcitrance'', one therefore obtains the qualitative relationship
\begin{equation}
\text{rate of change in intelligence} = \frac{\text{optimization power}}{\text{recalcitrance}}.
\label{eq:intelligenceRate}
\end{equation}
Therefore, decreasing recalcitrance increases the rate of change in intelligence of a given system. Hence, any upper bound on optimization power---whether it be due to genuine physical limitations or some other technological limitation---is not necessarily an insuperable obstruction to increasing the intelligence of a system, even if the general intelligence of the system is already comparable to the likes of Einstein. 

In consequence, the postulated mouse-Einstein intelligence scale seems na\"ively anthropocentric, which invariably leads one to drastically overestimate recalcitrance. \citeA{Yudkowsky08} puts this point as follows:
\begin{quote}
AI might make an apparently sharp jump in intelligence purely as the result of anthropomorphism, the human tendency to think of ``village idiot'' and ``Einstein'' as the extreme ends of the intelligence scale, instead of nearly indistinguishable points on the scale of minds-in-general. Everything dumber than a dumb human may appear to us as simply ``dumb''. One imagines the ``AI arrow'' creeping steadily up the scale of intelligence, moving past mice and chimpanzees, with AIs still remaining ``dumb'' because AIs cannot speak fluent language or write science papers, and then the AI arrow crosses the tiny gap from infra-idiot to ultra-Einstein in the course of one month or some similarly short period.
\end{quote}
Altogether, we take the preceding arguments as evidence that if an AGI is possible, then a superintelligence is also possible. Thus, given our arguments that there do in fact exist generally intelligent machines, there also exist superintelligent machines. This implies that $\SIM$ is a nontrivial trait.

Now, by the fact that $\SIM \subseteq \IM$, our suspicion that $\IM$ is $\TIME$-bounded (\conjref{conj:imtime}) entails the conjecture that $\SIM$ is similarly $\TIME$-bounded, and hence that $\SIM$ is similarly syntactic (where by ``similarly'' we mean ``with the same credence''). However, independent of whether $\IM$ is $\TIME$-bounded or not, we find it considerably more likely that $\SIM$ is $\TIME$-bounded. This follows because, by definition, a superintelligence can outperform humans in every intellectual activity, including intellectual activities that require speed, such as bullet chess. We therefore conjecture the following with a credence strictly greater than our credence for \conjref{conj:imtime}:
\begin{conjecture}
$\SIM$ is $\TIME$-bounded.
\label{conj:simtime}
\end{conjecture}

As before, if this conjecture is right, then it follows immediately that $\SIM$ is syntactic. Consequently, RTT cannot be used to decide whether an arbitrary machine is superintelligent or not. Thus, as far as we know, $\SIM$ may well be decidable.

Are superintelligent machines a trait we humans should desire? This issue is famously tricky, for on the one hand superintelligent machines promise to revolutionize healthcare and perhaps, in some sense, even ``solve'' human illnesses \cite{Sutskever23}, but on the other hand it seems possible, but in no way certain, that a superintelligent machine could, through no encoded malice aforethought in its underlying program, unintentionally wreak havoc on humanity \cite{Bostrom14}. There is ample literature, both fiction and non-fiction, about these two extreme outcomes, and it will take us too far afield to do justice to the arguments on both ends of the spectrum. The author lands in the camp that a superintelligence is indeed a natural desire, provided there are superintelligent machines that can also be contained, moral, and so forth. This brings us back to the set $\DIM$ of desired intelligent machines, which we now analyze in more detail.

\subsection{Is $\DIM$ Decidable?}

We believe the three traits $\CM, \MM,$ and $\SIM$ are all reasonable desires to simultaneously demand of the intelligent machines of the future. If we are right and these traits do indeed constitute part of the (partial) solution to the easy problem of desires, then, by \eref{eq:DIM}, $\DIM$ admits the decomposition
\begin{equation}
\DIM = \IM \cap \CM \cap \MM \cap \SIM \bigcirc_3 \T_4 \bigcirc _4 \cdots \bigcirc_{n-1} \T_n.
\label{eq:DIM2}
\end{equation}
This is to say that for $M$ to be desired, it is necessary that $M$ simultaneously be contained, moral, and superintelligent. As we stated before, for no reason other than hope (and also to consider the most interesting case), we assume $\DIM \neq \emptyset$ so that $\DIM$ is nontrivial.

From \eref{eq:DIM2}, one immediately obtains the the following three containments:
\begin{align}
\DIM &\subseteq \IM,\\
\DIM &\subseteq \MM,\\
\DIM &\subseteq \SIM.
\end{align}
Consequently, if \emph{any} one of Conjectures \ref{conj:imtime}-\ref{conj:simtime} is right---that is, if either $\IM, \MM$, or $\SIM$ is $\TIME$-bounded---then it holds that $\DIM$ is $\TIME$-bounded as well. The most general form of this suspicion is encapsulated in the following conjecture:
\begin{conjecture}
$\DIM$ is $\Phi$-bounded by a discriminating resource measure $\Phi$.
\end{conjecture}
In consequence, $\DIM$---assuming it is nontrivial---is necessarily a syntactic trait. As before, it therefore holds that RTT cannot determine whether $\DIM$ is decidable or not. Hence, RTT cannot determine whether the hard problem of desires is logically impossible to solve. Rather, toward proving this, it suffices to establish that both $\sem(\DIM) \neq \emptyset$ and $0 < |\syn(\DIM)| < \infty$ by our \thmref{thm:mainthm}. However, absent a (partial) solution to the easy problem of desires, this cannot obviously be done.

What is for sure, however, is that if all the traits we desire of intelligent machines are semantic (which could occur, for example, if Conjectures \ref{conj:imtime}-\ref{conj:simtime} are all false), then, by \lemref{lem:semantictraits} and \propref{cor:intersectiontraits}, $\DIM$ is necessarily semantic. In this case, $\DIM$ is undecidable \`a la RTT. Consequently, if we seek the formal means to be able to logically deduce, on pain of contradiction, whether an arbitrary machine is a desired intelligent machine or not, then we must at the very least insist in our list of desires of intelligent machines that not all of the machine traits be behavioral. That way, $\DIM$ becomes syntactic, and we are at the very least not subject to RTT.

\section{Discussion}
\label{sec:discussion}

Faced with the uncertain results of this paper, and in particular of our disabusing of the canonical notion that Rice's theorem applies to the question of whether or not a machine is intelligent, contained, moral, etc., one is naturally inclined to wonder if the formal decidability of traits of intelligence machines actually matters. This is a fair and important question, particularly since the problems we have considered assumed a paradigm in which we humans are essentially tasked with deciding whether or not an \emph{arbitrary} machine possesses a given trait, such as intelligence or morality. However, as researchers like \citeA{Yudkowsky08}, \citeA{Russell19}, and \citeA{Yampolskiy20} have correctly argued, this is not humanity's present position. In \citeS{Russell19} words:
\begin{quote}
The task is, fortunately, not the following: given a machine that possesses a high degree of intelligence, work out how to control it. If that were the task, we would be toast. A machine viewed as a black box, a fait accompli, might as well have arrived from outer space. And our chances of controlling a superintelligent entity from outer space are roughly zero. Similar arguments apply to methods of creating AI systems that guarantee we won't understand how they work.
\end{quote}
Instead, our present position affords us the cardinal power of deciding what and what does not go into the programs of the machines we build. In this sense, we have absolute control over our mechanistic progeny, at least in theory.

However, as noted by \citeA{Yampolskiy20}, the current state of affairs in AI research suggests that this theoretical ideal will become increasingly more difficult to obtain if progress continues as is. In part, this is due to influential researchers and powerful companies that are at the very least agnostic toward significant safety measures, perhaps because such measures invariably hamper progress (and \citeS{Ginsberg59} Moloch seldom rewards those who are last to a novel invention) or because of incorrigible philosophical priors to the effect that we humans cannot lose control of the machines we build.

More alarmingly, though, is the pathetic extent to which we currently understand how today's most advanced AI systems actually work. Indeed, in a multitude of respects, many modern deep learning models, for example, are genuine black boxes \cite{Rudin19} and concerns about their opacity abound \cite{Knight17,Lipton18,Savage22}. That said, it has been suggested that much of the power of contemporary models derives from their sheer inscrutability, as there seems to be a tradeoff between the prediction accuracy of a model and its interpretability \cite{Gunning19}. Of course, while this relationship is qualitative at best, we nevertheless find it an unnerving precedent. For, if it is right, then we are guaranteed that the increasingly powerful machines and algorithms of the future will necessarily be inscrutable black boxes---perhaps, even, to the extent that they may as well have descended from outer space \cite{Frank23}.

If this is indeed the future, then our results afford the semi-favorable perspective that there may well remain some logical means by which we could prove that the intelligent machines are in $\DIM$. Indeed, our results suggest that under no agreeable definition of ``desired intelligent machines'' will $\DIM$ be semantic, and hence no version of Rice's theorem can formally forbid one from deciding $\DIM$. We believe, therefore, that more research in this direction is paramount, particularly toward precisifying the definitions inherent to our desires for intelligent machines. For, if we fail to succeed at that, then no means will ever be found to decide traits like $\IM$, $\CM$, $\MM$, $\SIM$, and $\DIM$, and we shall be forever penned to a world in which we live on a \emph{belief} that our mechanistic progeny will eternally work in our favor. The author implores everyone to think whether that gamble is actually worth taking.

\acks{The author thanks Daniel Bashir, Mathew Bub, Charlie Cummings, Gabriel Golfetti, Chaitanya Karamchedu, Anna Kn\"orr, Brennen Quigley, Dan Sehayek, David Sobek, and Kaitlyn Zeichick for a myriad of stimulating discussions on artificial intelligence, computability theory, and their intersection. The author is especially grateful to Mathew Bub and Nikolas Claussen for providing many substantive remarks on an earlier draft of this paper.

This research was initiated by discussions at the Perimeter Institute for Theoretical Physics. Research at Perimeter Institute is supported in part by the Government of Canada through the Department of Innovation, Science and Economic Development Canada and by the Province of Ontario through the Ministry of Colleges and Universities.
}

\vskip 0.2in
\bibliography{references}

\begin{thebibliography}{}

\bibitem[\protect\BCAY{Achiam et~al.}{Achiam et~al.}{2023}]{Achiam23}
Achiam, J.\BBACOMMA\  et~al. \BBOP2023\BBCP.
\newblock \BBOQ {GPT-4} technical report\BBCQ\
\newblock {\Bem arXiv preprint arXiv:2303.08774}.

\bibitem[\protect\BCAY{Alfonseca et~al.}{Alfonseca et~al.}{2021}]{Alfonseca21}
Alfonseca, M.\BBACOMMA\  et~al. \BBOP2021\BBCP.
\newblock \BBOQ Superintelligence cannot be contained: Lessons from
  computability theory\BBCQ\
\newblock {\Bem J. Artif. Intell. Res.}, {\Bem 70}, 65--76.

\bibitem[\protect\BCAY{Allcott et~al.}{Allcott et~al.}{2020}]{Allcott20}
Allcott, H.\BBACOMMA\  et~al. \BBOP2020\BBCP.
\newblock \BBOQ The welfare effects of social media\BBCQ\
\newblock {\Bem Am Econ Rev}, {\Bem 110\/}(3), 629--676.

\bibitem[\protect\BCAY{Ambartsoumean\ \BBA\ Yampolskiy}{Ambartsoumean\ \BBA\
  Yampolskiy}{2023}]{Ambartsoumean23}
Ambartsoumean, V.~M.\BBACOMMA\  \BBA\ Yampolskiy, R.~V. \BBOP2023\BBCP.
\newblock \BBOQ {AI} risk skepticism, a comprehensive survey\BBCQ\
\newblock {\Bem arXiv preprint arXiv:2303.03885}.

\bibitem[\protect\BCAY{Amodei et~al.}{Amodei et~al.}{2016}]{Amodei16}
Amodei, D.\BBACOMMA\  et~al. \BBOP2016\BBCP.
\newblock \BBOQ Concrete problems in {AI} safety\BBCQ\
\newblock {\Bem arXiv preprint arXiv:1606.06565}.

\bibitem[\protect\BCAY{Arora\ \BBA\ Barak}{Arora\ \BBA\ Barak}{2009}]{Arora09}
Arora, S.\BBACOMMA\  \BBA\ Barak, B. \BBOP2009\BBCP.
\newblock {\Bem Computational Complexity: A Modern Approach\/} (1st \BEd).
\newblock Cambridge University Press.

\bibitem[\protect\BCAY{Asimov}{Asimov}{1950}]{Asimov50}
Asimov, I. \BBOP1950\BBCP.
\newblock {\Bem I, Robot\/} (1st \BEd).
\newblock Doubleday.

\bibitem[\protect\BCAY{Baum\ \BBA\ Sheehan}{Baum\ \BBA\ Sheehan}{1997}]{Baum97}
Baum, R.\BBACOMMA\  \BBA\ Sheehan, W. \BBOP1997\BBCP.
\newblock {\Bem In Search of Planet Vulcan, The Ghost of Newton's Clockwork
  Machine}.
\newblock Plenum Press.

\bibitem[\protect\BCAY{Benson}{Benson}{1998}]{Benson98}
Benson, H.~P. \BBOP1998\BBCP.
\newblock \BBOQ An outer approximation algorithm for generating all efficient
  extreme points in the outcome set of a multiple objective linear programming
  problem\BBCQ\
\newblock {\Bem J Glob Optim}, {\Bem 13\/}(1), 1--24.

\bibitem[\protect\BCAY{Blum}{Blum}{1967}]{Blum67}
Blum, M. \BBOP1967\BBCP.
\newblock \BBOQ A machine-independent theory of the complexity of recursive
  functions\BBCQ\
\newblock {\Bem J. ACM}, {\Bem 14\/}(2), 322–336.

\bibitem[\protect\BCAY{Bostrom}{Bostrom}{2011}]{Bostrom11}
Bostrom, N. \BBOP2011\BBCP.
\newblock \BBOQ Information hazards: A typology of potential harms from
  knowledge\BBCQ\
\newblock {\Bem Rev. Contemp. Philos.}, {\Bem 10}, 44--79.

\bibitem[\protect\BCAY{Bostrom}{Bostrom}{2014}]{Bostrom14}
Bostrom, N. \BBOP2014\BBCP.
\newblock {\Bem Superintelligence: Paths, Dangers, and Strategies}.
\newblock Oxford University Press.

\bibitem[\protect\BCAY{Bubeck et~al.}{Bubeck et~al.}{2023}]{Bubeck23}
Bubeck, S.\BBACOMMA\  et~al. \BBOP2023\BBCP.
\newblock \BBOQ Sparks of artificial general intelligence: Early experiments
  with {GPT-4}\BBCQ\
\newblock {\Bem arXiv preprint arXiv:2303.12712}.

\bibitem[\protect\BCAY{Chalmers}{Chalmers}{2010}]{Chalmers10}
Chalmers, D. \BBOP2010\BBCP.
\newblock \BBOQ The singularity: A philosophical analysis\BBCQ\
\newblock {\Bem J. Conscious Stud.}, {\Bem 17}, 7--65.

\bibitem[\protect\BCAY{Chollet}{Chollet}{2019}]{Chollet19}
Chollet, F. \BBOP2019\BBCP.
\newblock \BBOQ On the measure of intelligence\BBCQ\
\newblock {\Bem arXiv preprint arXiv:1911.01547}.

\bibitem[\protect\BCAY{Corwin}{Corwin}{2002}]{Corwin02}
Corwin, J. \BBOP2002\BBCP.
\newblock \BBOQ {AI} boxing\BBCQ\
\newblock Online.
\newblock Accessed 12 December 2023.

\bibitem[\protect\BCAY{Davenport\ \BBA\ Kalakota}{Davenport\ \BBA\
  Kalakota}{2019}]{Davenport19}
Davenport, T.\BBACOMMA\  \BBA\ Kalakota, R. \BBOP2019\BBCP.
\newblock \BBOQ The potential for artificial intelligence in healthcare\BBCQ\
\newblock {\Bem Future Healthc J.}, {\Bem 6}, 94--98.

\bibitem[\protect\BCAY{Degrave et~al.}{Degrave et~al.}{2022}]{Degrave22}
Degrave, J.\BBACOMMA\  et~al. \BBOP2022\BBCP.
\newblock \BBOQ Magnetic control of tokamak plasmas through deep reinforcement
  learning\BBCQ\
\newblock {\Bem Nature}, {\Bem 602}, 414--419.

\bibitem[\protect\BCAY{Deutsch}{Deutsch}{1985}]{Deutsch85}
Deutsch, D. \BBOP1985\BBCP.
\newblock \BBOQ Quantum theory, the {Church-Turing} principle and the universal
  quantum computer\BBCQ\
\newblock {\Bem Proc. R. Soc. Lond. A}, {\Bem 400}, 97--117.

\bibitem[\protect\BCAY{Fawzi et~al.}{Fawzi et~al.}{2022}]{Fawzi22}
Fawzi, A.\BBACOMMA\  et~al. \BBOP2022\BBCP.
\newblock \BBOQ Discovering faster matrix multiplication algorithms with
  reinforcement learning\BBCQ\
\newblock {\Bem Nature}, {\Bem 610}, 47--53.

\bibitem[\protect\BCAY{Frank}{Frank}{2023}]{Frank23}
Frank, M.~C. \BBOP2023\BBCP.
\newblock \BBOQ Baby steps in evaluating the capacities of large language
  models\BBCQ\
\newblock {\Bem Nat Rev Psychol}, {\Bem 2}, 451--452.

\bibitem[\protect\BCAY{Fry\ \BBA\ Hale}{Fry\ \BBA\ Hale}{1996}]{Fry96}
Fry, A.~E.\BBACOMMA\  \BBA\ Hale, S. \BBOP1996\BBCP.
\newblock \BBOQ Processing speed, working memory, and fluid intelligence:
  Evidence for a developmental cascade\BBCQ\
\newblock {\Bem Psychol Sci}, {\Bem 7}, 237--241.

\bibitem[\protect\BCAY{Ginsberg}{Ginsberg}{1959}]{Ginsberg59}
Ginsberg, A. \BBOP1959\BBCP.
\newblock {\Bem Howl and Other Poems\/} (Reissue \BEd).
\newblock City Lights Publishers.

\bibitem[\protect\BCAY{Good}{Good}{1966}]{Good66}
Good, I.~J. \BBOP1966\BBCP.
\newblock \BBOQ Speculations concerning the first ultraintelligent
  machine\BBCQ\
\newblock {\Bem Adv. Comput.}, {\Bem 6}, 31--88.

\bibitem[\protect\BCAY{Grace et~al.}{Grace et~al.}{2024}]{Grace24}
Grace, K.\BBACOMMA\  et~al. \BBOP2024\BBCP.
\newblock \BBOQ Thousands of {AI} authors on the future of {AI}\BBCQ\
\newblock {\Bem arXiv preprint arXiv:2401.02843}.

\bibitem[\protect\BCAY{Gunning\ \BBA\ Aha}{Gunning\ \BBA\
  Aha}{2019}]{Gunning19}
Gunning, D.\BBACOMMA\  \BBA\ Aha, D. \BBOP2019\BBCP.
\newblock \BBOQ {DARPA}'s explainable artificial intelligence ({XAI})
  program\BBCQ\
\newblock {\Bem AI Magazine}, {\Bem 40}, 44--58.

\bibitem[\protect\BCAY{Hadfield-Menell et~al.}{Hadfield-Menell
  et~al.}{2017}]{Hadfield17}
Hadfield-Menell, D.\BBACOMMA\  et~al. \BBOP2017\BBCP.
\newblock \BBOQ The off-switch game\BBCQ\
\newblock In {\Bem Proceedings of the 26th International Joint Conference on
  Artificial Intelligence}.

\bibitem[\protect\BCAY{Haidt}{Haidt}{2001}]{Haidt01}
Haidt, J. \BBOP2001\BBCP.
\newblock \BBOQ The emotional dog and its rational tail\BBCQ\
\newblock {\Bem Psychol Rev}, {\Bem 104}, 814--834.

\bibitem[\protect\BCAY{Haier}{Haier}{2016}]{Haier16}
Haier, R.~J. \BBOP2016\BBCP.
\newblock {\Bem The Neuroscience of Intelligence\/} (1st \BEd).
\newblock Cambridge University Press.

\bibitem[\protect\BCAY{Hawking, Russell, Tegmark,\ \BBA\ Wilczek}{Hawking
  et~al.}{2014}]{Hawking14}
Hawking, S., Russell, S., Tegmark, M., \BBA\ Wilczek, F. \BBOP2014\BBCP.
\newblock \BBOQ Stephen {H}awking: `{T}ranscendence looks at the implications
  of artificial intelligence---but are we taking {AI} seriously enough?'\BBCQ\
\newblock Online.
\newblock Accessed 13 October 2023.

\bibitem[\protect\BCAY{Hendrycks et~al.}{Hendrycks et~al.}{2022}]{Hendrycks22}
Hendrycks, D.\BBACOMMA\  et~al. \BBOP2022\BBCP.
\newblock \BBOQ Unsolved problems in {ML} safety\BBCQ\
\newblock {\Bem arXiv preprint arXiv:2109.13916}.

\bibitem[\protect\BCAY{Hofstadter\ \BBA\ Dennett}{Hofstadter\ \BBA\
  Dennett}{2001}]{Hofstadter01}
Hofstadter, D.\BBACOMMA\  \BBA\ Dennett, D. \BBOP2001\BBCP.
\newblock {\Bem The Mind's I: Fantasies and Reflections on Self and Soul\/}
  (1st \BEd).
\newblock Basic Books.

\bibitem[\protect\BCAY{Huang et~al.}{Huang et~al.}{2023}]{Huang23}
Huang, L.\BBACOMMA\  et~al. \BBOP2023\BBCP.
\newblock \BBOQ A survey on hallucination in large language models: Principles,
  taxonomy, challenges, and open questions\BBCQ\
\newblock {\Bem arXiv preprint arXiv:2311.05232}.

\bibitem[\protect\BCAY{Johnson}{Johnson}{1997}]{Johnson97}
Johnson, G. \BBOP1997\BBCP.
\newblock \BBOQ To test a powerful computer, play an ancient game\BBCQ\
\newblock Online.
\newblock Accessed 03 January 2024.

\bibitem[\protect\BCAY{Jumper et~al.}{Jumper et~al.}{2021}]{Jumper21}
Jumper, J.\BBACOMMA\  et~al. \BBOP2021\BBCP.
\newblock \BBOQ Highly accurate protein structure prediction with
  {AlphaFold}\BBCQ\
\newblock {\Bem Nature}, {\Bem 596}, 583--589.

\bibitem[\protect\BCAY{Kamalov et~al.}{Kamalov et~al.}{2023}]{Kamalov23}
Kamalov, F.\BBACOMMA\  et~al. \BBOP2023\BBCP.
\newblock \BBOQ New era of artificial intelligence in education: Towards a
  sustainable multifaceted revolution\BBCQ\
\newblock {\Bem Sustainability}, {\Bem 15}, 12451.

\bibitem[\protect\BCAY{Kleene}{Kleene}{1952}]{Kleene52}
Kleene, S.~C. \BBOP1952\BBCP.
\newblock {\Bem Introduction to Metamathematics\/} (1st \BEd).
\newblock P. Noordhoff N.V., Groningen.

\bibitem[\protect\BCAY{Knight}{Knight}{2017}]{Knight17}
Knight, W. \BBOP2017\BBCP.
\newblock \BBOQ The dark secret at the heart of {AI}\BBCQ\
\newblock Online.
\newblock Accessed 24 December 2023.

\bibitem[\protect\BCAY{Kozen}{Kozen}{2006}]{Kozen06}
Kozen, D. \BBOP2006\BBCP.
\newblock {\Bem Theory of Computation}.
\newblock Springer-Verlag, London.

\bibitem[\protect\BCAY{Lampson}{Lampson}{1973}]{Lampson73}
Lampson, B.~W. \BBOP1973\BBCP.
\newblock \BBOQ A note on the confinement problem\BBCQ\
\newblock {\Bem Commun. ACM}, {\Bem 16\/}(10), 613--615.

\bibitem[\protect\BCAY{Leben}{Leben}{2018}]{Leben18}
Leben, D. \BBOP2018\BBCP.
\newblock {\Bem Ethics for Robots: How to Design a Moral Algorithm\/} (1st
  \BEd).
\newblock Routledge.

\bibitem[\protect\BCAY{Lipton}{Lipton}{2018}]{Lipton18}
Lipton, Z.~C. \BBOP2018\BBCP.
\newblock \BBOQ The mythos of model interpretability: In machine learning, the
  concept of interpretability is both important and slippery.\BBCQ\
\newblock {\Bem Queue}, {\Bem 16\/}(3), 31–57.

\bibitem[\protect\BCAY{Madhavan}{Madhavan}{2023}]{Madhavan23}
Madhavan, A. \BBOP2023\BBCP.
\newblock \BBOQ Brain-inspired computing can help us create faster, more
  energy-efficient devices---if we win the race\BBCQ\
\newblock Online.
\newblock Accessed 16 January 2024.

\bibitem[\protect\BCAY{Marshall}{Marshall}{2023}]{Marshall23}
Marshall, A. \BBOP2023\BBCP.
\newblock \BBOQ {GM's} cruise loses its self-driving license in {S}an
  {F}rancisco after a robotaxi dragged a person\BBCQ\
\newblock Online.
\newblock Accessed 20 December 2023.

\bibitem[\protect\BCAY{McCarthy}{McCarthy}{1980}]{McCarthy80}
McCarthy, J. \BBOP1980\BBCP.
\newblock \BBOQ Circumscription---a form of non-monotonic reasoning\BBCQ\
\newblock {\Bem Artificial Intelligence}, {\Bem 13}, 27--39.

\bibitem[\protect\BCAY{McCorduck}{McCorduck}{1979}]{McCorduck79}
McCorduck, P. \BBOP1979\BBCP.
\newblock {\Bem Machines Who Think: A Personal Inquiry into the History and
  Prospects of Artificial Intelligence\/} (1st \BEd).
\newblock W. H. Freeman.

\bibitem[\protect\BCAY{Moravec}{Moravec}{1988}]{Moravec88}
Moravec, H. \BBOP1988\BBCP.
\newblock {\Bem Mind Children\/} (1st \BEd).
\newblock Harvard University Press.

\bibitem[\protect\BCAY{Nielsen\ \BBA\ Chuang}{Nielsen\ \BBA\
  Chuang}{2011}]{Nielsen11}
Nielsen, M.~A.\BBACOMMA\  \BBA\ Chuang, I.~L. \BBOP2011\BBCP.
\newblock {\Bem Quantum Computation and Quantum Information\/} (Anniversary
  \BEd).
\newblock Cambridge University Press.

\bibitem[\protect\BCAY{Quine}{Quine}{1951}]{Quine51}
Quine, W.~V. \BBOP1951\BBCP.
\newblock \BBOQ Two dogmas of empiricism\BBCQ\
\newblock {\Bem Philos Rev}, {\Bem 60\/}(1), 20--43.

\bibitem[\protect\BCAY{Rice}{Rice}{1953}]{Rice53}
Rice, H.~G. \BBOP1953\BBCP.
\newblock \BBOQ Classes of recursively enumerable sets and their decision
  problems\BBCQ\
\newblock {\Bem Trans. Am. Math. Soc.}, {\Bem 74}, 358--366.

\bibitem[\protect\BCAY{Robi{\^{c}}}{Robi{\^{c}}}{2020}]{Robic20}
Robi{\^{c}}, B. \BBOP2020\BBCP.
\newblock {\Bem The Foundations of Computability Theory\/} (2nd \BEd).
\newblock Springer.

\bibitem[\protect\BCAY{Rudin\ \BBA\ Radin}{Rudin\ \BBA\ Radin}{2019}]{Rudin19}
Rudin, C.\BBACOMMA\  \BBA\ Radin, J. \BBOP2019\BBCP.
\newblock \BBOQ Why are we using black box models in {AI} when we don't need
  to? a lesson from an explainable {AI} competition\BBCQ\
\newblock Online.
\newblock Accessed 24 December 2023.

\bibitem[\protect\BCAY{Russell}{Russell}{1972}]{Russell72}
Russell, B. \BBOP1972\BBCP.
\newblock {\Bem The philosophy of logical atomism\/} (1st \BEd).
\newblock Fontana.

\bibitem[\protect\BCAY{Russell}{Russell}{2017}]{Russell17}
Russell, S. \BBOP2017\BBCP.
\newblock \BBOQ Provably beneficial artificial intelligence\BBCQ\
\newblock In {\Bem The Next Step: Exponential Life}. BBVA OpenMind.

\bibitem[\protect\BCAY{Russell}{Russell}{2019}]{Russell19}
Russell, S. \BBOP2019\BBCP.
\newblock {\Bem Human Compatible: Artificial Intelligence and the Problem of
  Control\/} (1st \BEd).
\newblock Penguin Books.

\bibitem[\protect\BCAY{Russell}{Russell}{2022}]{Russell22}
Russell, S. \BBOP2022\BBCP.
\newblock {\Bem Artificial Intelligence and the Problem of Control}, \BPGS\
  19--24.
\newblock Springer International Publishing, Cham.

\bibitem[\protect\BCAY{Russell\ \BBA\ Norvig}{Russell\ \BBA\
  Norvig}{2020}]{Russell20}
Russell, S.\BBACOMMA\  \BBA\ Norvig, P. \BBOP2020\BBCP.
\newblock {\Bem Artificial Intelligence: A Modern Approach\/} (4th \BEd).
\newblock Pearson Education, Inc.

\bibitem[\protect\BCAY{Savage}{Savage}{2022}]{Savage22}
Savage, N. \BBOP2022\BBCP.
\newblock \BBOQ Breaking into the black box of artificial intelligence\BBCQ\
\newblock Online.
\newblock Accessed 24 December 2023.

\bibitem[\protect\BCAY{Searle}{Searle}{1980}]{Searle80}
Searle, J.~R. \BBOP1980\BBCP.
\newblock \BBOQ Minds, brains, and programs\BBCQ\
\newblock {\Bem Behavioral and Brain Sciences}, {\Bem 3\/}(3), 417–424.

\bibitem[\protect\BCAY{Sevilla et~al.}{Sevilla et~al.}{2022}]{Sevilla22}
Sevilla, J.\BBACOMMA\  et~al. \BBOP2022\BBCP.
\newblock \BBOQ Compute trends across three eras of machine learning\BBCQ\
\newblock In {\Bem 2022 International Joint Conference on Neural Networks
  (IJCNN)}, \BPGS\ 1--8.

\bibitem[\protect\BCAY{Sipser}{Sipser}{2013}]{Sipser13}
Sipser, M. \BBOP2013\BBCP.
\newblock {\Bem Introduction to the Theory of Computation\/} (3rd \BEd).
\newblock Cengage Learning.

\bibitem[\protect\BCAY{Strogatz}{Strogatz}{2022}]{Strogatz22}
Strogatz, S. \BBOP2022\BBCP.
\newblock \BBOQ Can computers be mathematicians?\BBCQ\
\newblock Online.
\newblock Accessed 24 January 2024.

\bibitem[\protect\BCAY{Sutskever}{Sutskever}{2023}]{Sutskever23}
Sutskever, I. \BBOP2023\BBCP.
\newblock \BBOQ The exciting, perilous journey toward {AGI}\BBCQ\
\newblock Online.
\newblock Accessed 17 November 2023.

\bibitem[\protect\BCAY{Tabbaa}{Tabbaa}{2018}]{Tabbaa18}
Tabbaa, B. \BBOP2018\BBCP.
\newblock \BBOQ The rise and fall of {K}night {C}apital---buy high, sell low.
  {R}inse and repeat\BBCQ\
\newblock Online.
\newblock Accessed 20 December 2023.

\bibitem[\protect\BCAY{Tegmark}{Tegmark}{2017}]{Tegmark17}
Tegmark, M. \BBOP2017\BBCP.
\newblock {\Bem Life 3.0: Being Human in the Age of Artificial Intelligence\/}
  (1st \BEd).
\newblock Knopf.

\bibitem[\protect\BCAY{Turing}{Turing}{1936}]{Turing36}
Turing, A.~M. \BBOP1936\BBCP.
\newblock \BBOQ On computable numbers, with an application to the
  entscheidungproblem\BBCQ\
\newblock {\Bem Proc. Lond. Math. Soc.}, {\Bem 42}, 230--265.

\bibitem[\protect\BCAY{Turing}{Turing}{1948}]{Turing48}
Turing, A.~M. \BBOP1948\BBCP.
\newblock \BBOQ Intelligent machinery\BBCQ\
\newblock {\Bem National Physical Laboratory}.

\bibitem[\protect\BCAY{Turing}{Turing}{1950}]{Turing50}
Turing, A.~M. \BBOP1950\BBCP.
\newblock \BBOQ Computing machinery and intelligence\BBCQ\
\newblock {\Bem Mind}, {\Bem 59\/}(236), 433--446.

\bibitem[\protect\BCAY{Turing}{Turing}{1954}]{Turing54}
Turing, A.~M. \BBOP1954\BBCP.
\newblock \BBOQ Solvable and unsolvable problems\BBCQ\
\newblock {\Bem Science News}, {\Bem 31}, 7--23.

\bibitem[\protect\BCAY{Wallach\ \BBA\ Allen}{Wallach\ \BBA\
  Allen}{2010}]{Wallach10}
Wallach, W.\BBACOMMA\  \BBA\ Allen, C. \BBOP2010\BBCP.
\newblock {\Bem Moral Machines: Teaching Robots Right from Wrong\/} (1st \BEd).
\newblock Oxford University Press.

\bibitem[\protect\BCAY{Wilbur et~al.}{Wilbur et~al.}{2023}]{Wilbur23}
Wilbur, M.\BBACOMMA\  et~al. \BBOP2023\BBCP.
\newblock \BBOQ Artificial intelligence for smart transportation\BBCQ\
\newblock {\Bem arXiv preprint arXiv:2308.07457}.

\bibitem[\protect\BCAY{Yakimova\ \BBA\ Ojamo}{Yakimova\ \BBA\
  Ojamo}{2023}]{Yakimova23}
Yakimova, Y.\BBACOMMA\  \BBA\ Ojamo, J. \BBOP2023\BBCP.
\newblock \BBOQ Artificial intelligence act: deal on comprehensive rules for
  trustworthy {AI}\BBCQ\
\newblock Online.
\newblock Accessed 28 December 2023.

\bibitem[\protect\BCAY{Yampolskiy}{Yampolskiy}{2012}]{Yampolskiy12}
Yampolskiy, R.~V. \BBOP2012\BBCP.
\newblock \BBOQ Leakproofing the singularity: Artificial intelligence
  confinement problem\BBCQ\
\newblock {\Bem J. Conscious Stud.}, {\Bem 19}, 194--214.

\bibitem[\protect\BCAY{Yampolskiy}{Yampolskiy}{2020}]{Yampolskiy20}
Yampolskiy, R.~V. \BBOP2020\BBCP.
\newblock \BBOQ On controllability of {AI}\BBCQ\
\newblock {\Bem arXiv preprint arXiv:2008.04071}.

\bibitem[\protect\BCAY{Yudkowsky}{Yudkowsky}{2002}]{Yudkowsky02}
Yudkowsky, E. \BBOP2002\BBCP.
\newblock \BBOQ The {AI}-box experiment\BBCQ\
\newblock Online.
\newblock Accessed 12 December 2023.

\bibitem[\protect\BCAY{Yudkowsky}{Yudkowsky}{2008}]{Yudkowsky08}
Yudkowsky, E. \BBOP2008\BBCP.
\newblock \BBOQ Artificial intelligence as a positive and negative factor in
  global risk\BBCQ\
\newblock In Bostrom, N.\BBACOMMA\  \BBA\ Cirkovic, M.~M.\BEDS, {\Bem Global
  Catastrophic Risks}, \BPGS\ 308--345. Oxford University Press.

\end{thebibliography}
\bibliographystyle{theapa}

\end{document}